\documentclass[twocolumn,twoside]{IEEEtran}

\usepackage{etex,wrapfig}
\reserveinserts{28}
\usepackage{setspace,cite}

\usepackage[dvips]{graphicx}
\usepackage{multirow,multicol}
\usepackage[table]{xcolor}
\usepackage{ctable}

\usepackage[tbtags]{amsmath}
\usepackage{amsbsy}
\usepackage{amssymb}
\usepackage{amsfonts}

\usepackage{graphicx}
\usepackage{pstricks}
\usepackage{pst-node}
\usepackage{pstricks-add}
\usepackage{url}
\usepackage{algorithmic}
\usepackage{algorithm}
\usepackage{balance}
\usepackage{rotating}
\usepackage{multirow}
\usepackage{subcaption}
\usepackage{fancyvrb}
\usepackage{latexsym}
\usepackage{verbatim}

\newtheorem{theorem}{Theorem}
\newtheorem{proposition}{Proposition}
\newtheorem{lemma}{Lemma}

\newtheorem{remark}{Remark}

\newtheorem{definition}{Definition}
\newtheorem{example}{Example}

\newtheorem{assumption}{Assumption}

\newenvironment{densitemize}
{\begin{list}               
    {$\bullet$ \hfill}{
        \setlength{\leftmargin}{\parindent}
        \setlength{\parsep}{0.04\baselineskip}
        \setlength{\itemsep}{0.5\parsep}
        \setlength{\labelwidth}{\leftmargin}
        \setlength{\labelsep}{0em}}
    }
{\end{list}}

\providecommand{\eref}[1]{\eqref{#1}}  
\providecommand{\cref}[1]{Chapter~\ref{#1}}

\providecommand{\fref}[1]{Figure~\ref{#1}}

\providecommand{\R}{\ensuremath{\mathbb{R}}}

\providecommand{\E}{\ensuremath{\mathbb{E}}}

\providecommand{\bydef}{\overset{\text{def}}{=}}

\renewcommand{\vec}[1]{\ensuremath{\boldsymbol{#1}}}
\providecommand{\mat}[1]{\ensuremath{\boldsymbol{#1}}}

\providecommand{\calD}{\mathcal{D}}

\providecommand{\calF}{\mathcal{F}}

\providecommand{\calI}{\mathcal{I}}

\providecommand{\calL}{\mathcal{L}}

\providecommand{\calN}{\mathcal{N}}

\providecommand{\mA}{\mat{A}}
\providecommand{\mB}{\mat{B}}

\providecommand{\mG}{\mat{G}}
\providecommand{\mH}{\mat{H}}
\providecommand{\mI}{\mat{I}}

\providecommand{\mS}{\mat{S}}

\providecommand{\mW}{\mat{W}}

\providecommand{\vb}{\vec{b}}

\providecommand{\vs}{\vec{s}}

\providecommand{\vu}{\vec{u}}
\providecommand{\vv}{\vec{v}}

\providecommand{\vx}{\vec{x}}
\providecommand{\vy}{\vec{y}}

\providecommand{\mTheta}{\mat{\Theta}}

\providecommand{\vepsilon}{\vec{\epsilon}}

\providecommand{\vtheta}{\vec{\theta}}

\providecommand{\Htilde}{\widetilde{H}}

\providecommand{\mWtilde}{\mat{\widetilde{W}}}
\providecommand{\mHtilde}{\mat{\widetilde{H}}}

\providecommand{\vvtilde}{\boldsymbol{\widetilde{v}}}
\providecommand{\vxtilde}{\boldsymbol{\widetilde{x}}}

\providecommand{\vxhat}{\boldsymbol{\widehat{x}}}

\providecommand{\vvhat}{\boldsymbol{\widehat{v}}}
\providecommand{\vubar}{\boldsymbol{\bar{u}}}

\providecommand{\vzero}{\vec{0}}
\providecommand{\vone}{\vec{1}}

\newcommand{\subjectto}{\mathop{\mathrm{subject\, to}}}
\newcommand{\argmin}[1]{\mathop{\underset{#1}{\mbox{argmin}}}}
\newcommand{\argmax}[1]{\mathop{\underset{#1}{\mbox{argmax}}}}

\newcommand{\diag}[1]{\mathop{\mathrm{diag}\left\{#1\right\}}}

\input{customFigureMacros}
\usepackage{color,colortbl,hhline,makecell}
\definecolor{Gray}{gray}{0.9}
\definecolor{White}{rgb}{1,1,1}

\title{Plug-and-Play ADMM for Image Restoration: \\ Fixed Point Convergence and Applications}
\author{Stanley H. Chan,~\IEEEmembership{Member,~IEEE}, Xiran Wang,~\IEEEmembership{Student Member,~IEEE}, Omar A. Elgendy,~\IEEEmembership{Student Member,~IEEE} \vspace{-2ex}
\thanks{The authors are with the School of Electrical and Computer Engineering, Purdue University, West Lafayette, IN 47907, USA. S. Chan is also with the Department of Statistics, Purdue University, West Lafayette, IN 47907, USA. Emails: \texttt{ \{stanchan, wang470, oelgendy\}@purdue.edu}.}
\thanks{This paper follows the concept of reproducible research. All the results and examples presented in the paper are reproducible using the code and images available online at http://engineering.purdue.edu/ChanGroup/.}
}

\graphicspath{{./pix/}}

\begin{document}
\maketitle

\begin{abstract}
Alternating direction method of multiplier (ADMM) is a widely used algorithm for solving constrained optimization problems in image restoration. Among many useful features, one critical feature of the ADMM algorithm is its \emph{modular} structure which allows one to plug in any off-the-shelf image denoising algorithm for a subproblem in the ADMM algorithm. Because of the plug-in nature, this type of ADMM algorithms is coined the name ``Plug-and-Play ADMM''. Plug-and-Play ADMM has demonstrated promising empirical results in a number of recent papers. However, it is unclear under what conditions and by using what denoising algorithms would it guarantee convergence. Also, since Plug-and-Play ADMM uses a specific way to split the variables, it is unclear if fast implementation can be made for common Gaussian and Poissonian image restoration problems.

In this paper, we propose a Plug-and-Play ADMM algorithm with provable fixed point convergence. We show that for any denoising algorithm satisfying an asymptotic criteria, called \emph{bounded denoisers}, Plug-and-Play ADMM converges to a fixed point under a continuation scheme. We also present fast implementations for two image restoration problems on super-resolution and single-photon imaging. We compare Plug-and-Play ADMM with state-of-the-art algorithms in each problem type, and demonstrate promising experimental results of the algorithm.
\end{abstract}


\section{Introduction}
\subsection{MAP and ADMM}
Many image restoration tasks can be posted as the following inverse problem: Given an observed image $\vy \in \R^n$ corrupted according to some forward model and noise, find the underlying image $\vx \in \R^n$ which ``best explains'' the observation. In estimation, we often formulate this problem as a maximum-a-posteriori (MAP) estimation \cite{Poor_1998}, where the goal is to maximize the posterior probability:
\begin{align}
\vxhat
&= \argmax{\vx} \;\; p(\vx \;|\; \vy) \notag \\
&= \argmin{\vx} \;\; -\log p(\vy \;|\; \vx) - \log p(\vx),
\label{eq:MAP formulation}
\end{align}
for some conditional probability $p(\vy \,|\, \vx)$ defining the forward imaging model, and a prior distribution $p(\vx)$ defining the probability distribution of the latent image. Because of the explicit use of the forward and the prior models, MAP estimation is also a model-based image reconstruction (MBIR) method \cite{Bouman_2015} which has many important applications in deblurring \cite{Afonso_Bioucas-Dias_Figueiredo_2010,Chan_Khoshabeh_Gibson_2011,Yang_Zhang_Yin_2009}, interpolation \cite{Dahl_Hansen_Jensen_2010,Garcia_2010,Zhou_Chen_Ren_2009}, super-resolution \cite{Dong_Loy_He_2014,Peleg_Elad_2014,He_Siu_2011,Yang_Wright_Huang_2008} and computed tomography \cite{Sreehari_Venkatakrishnan_Wohlberg_2015}, to name a few.

It is not difficult to see that solving the MAP problem in \eref{eq:MAP formulation} is equivalent to solving an optimization problem
\begin{equation}
\vxhat = \argmin{\vx} \;\; f(\vx) + \lambda g(\vx),
\label{eq:MAP equivalent}
\end{equation}
with $f(\vx) \bydef -\log p(\vy \,|\, \vx)$ and $g(\vx) \bydef -(1/\lambda) \log p(\vx)$. The optimization in \eref{eq:MAP equivalent} is a generic unconstrained optimization. Thus, standard optimization algorithms can be used to solve the problem. In this paper, we focus on the alternating direction method of multiplier (ADMM) \cite{Boyd_Parikh_Chu_Peleato_Eckstein_2011}, which has become the workhorse for a variety of problems in the form of \eref{eq:MAP equivalent}.

The idea of ADMM is to convert \eref{eq:MAP equivalent}, an unconstrained optimization, into a constrained problem
\begin{equation}
(\vxhat,\vvhat) = \argmin{\vx,\vv} \;\; f(\vx) + \lambda g(\vv), \;\; \subjectto \;\; \vx = \vv,
\label{eq:MAP equivalent 2}
\end{equation}
and consider its augmented Lagrangian function:
\begin{equation}
\calL(\vx,\vv,\vu) = f(\vx) + \lambda g(\vv) + \vu^T(\vx-\vv) + \frac{\rho}{2}\|\vx - \vv\|^2.
\label{eq:augmented lagrangian function}
\end{equation}
The minimizer of \eref{eq:MAP equivalent 2} is then the saddle point of $\calL$, which can be found by solving a sequence of subproblems
\begin{align}
\vx^{(k+1)} &= \argmin{\vx\in \R^n} \;\; f(\vx) +\frac{\rho}{2} \|\vx - \vxtilde^{(k)}\|^2, \label{eq:ADMM2,x}\\
\vv^{(k+1)} &= \argmin{\vv\in \R^n} \;\; \lambda g(\vv) + \frac{\rho}{2}\|\vv - \vvtilde^{(k)}\|^2,\label{eq:ADMM2,v}\\
\vubar^{(k+1)} &= \vubar^{(k)} + (\vx^{(k+1)} - \vv^{(k+1)}),\label{eq:ADMM2,u}
\end{align}
where $\vubar^{(k)} \bydef (1/\rho)\vu^{(k)}$ is the scaled Lagrange multiplier, $\vxtilde^{(k)} \bydef \vv^{(k)}-\vubar^{(k)}$ and $\vvtilde^{(k)} \bydef \vx^{(k+1)}+\vubar^{(k)}$. Under mild conditions, e.g., when both $f$ and $g$ are closed, proper and convex, and if a saddle point of $\calL$ exists, one can show that the iterates \eref{eq:ADMM2,x}-\eref{eq:ADMM2,u} converge to the solution of \eref{eq:MAP equivalent 2} (See \cite{Boyd_Parikh_Chu_Peleato_Eckstein_2011} for details).

\subsection{Plug-and-Play ADMM}
An important feature of the ADMM iterations \eref{eq:ADMM2,x}-\eref{eq:ADMM2,u} is its modular structure. In particular, \eref{eq:ADMM2,x} can be regarded as an inversion step as it involves the forward imaging model $f(\vx)$, whereas \eref{eq:ADMM2,v} can be regarded as a denoising step as it involves the prior $g(\vv)$. To see the latter, if we define $\sigma = \sqrt{\lambda/\rho}$, it is not difficult to show that \eref{eq:ADMM2,v} is
\begin{equation}
\vv^{(k+1)} = \argmin{\vv\in \R^n} \;\; g(\vv) + \frac{1}{2\sigma^2}\|\vv - \vvtilde^{(k)}\|^2.
\label{eq:ADMM2,v rewrite}
\end{equation}
Treating $\vvtilde^{(k)}$ as the ``noisy'' image, \eref{eq:ADMM2,v rewrite} minimizes the residue between $\vvtilde^{(k)}$ and the ``clean'' image $\vv$ using the prior $g(\vv)$. For example, if $g(\vx) = \|\vx\|_{TV}$ (the total variation norm), then \eref{eq:ADMM2,v rewrite} is the standard total variation denoising problem.

Building upon this intuition, Venkatakrishnan et al. \cite{Venkatakrishnan_Bouman_Wohlberg_2013} proposed a variant of the ADMM algorithm by suggesting that one does not need to specify $g$ before running the ADMM. Instead, they replace \eref{eq:ADMM2,v} by using an off-the-shelf image denoising algorithm, denoted by $\calD_{\sigma}$, to yield
\begin{equation}
\vv^{(k+1)} = \calD_\sigma\left( \vvtilde^{(k)} \right).
\label{eq:ADMM2,v rewrite 2}
\end{equation}
Because of the heuristic nature of the method, they called the resulting algorithm as the Plug-and-Play ADMM. An interesting observation they found in \cite{Venkatakrishnan_Bouman_Wohlberg_2013} is that although Plug-and-Play ADMM appears ad-hoc, for a number of image reconstruction problems the algorithm indeed performs better than some state-of-the-art methods. A few recent reports have concurred similar observations \cite{Sreehari_Venkatakrishnan_Wohlberg_2015,Dar_Bruckstein_Elad_2015,Rond_Giryes_Elad_2015,Brifman_Romano_Elad_2016}.

\subsection{Challenges of Plug-and-Play ADMM}
\label{sec:problems}
From a theoretical point of view, the main challenge of analyzing Plug-and-Play ADMM is the denoiser $\calD_\sigma$. Since $\calD_\sigma$ is often nonlinear and does not have closed form expressions, the analysis has been very difficult. Specifically, the following three questions remain open:
\begin{enumerate}
\item Convergence of the Algorithm. Classical results of ADMM require $g$ to be closed, proper and convex in order to ensure convergence \cite{Boyd_Parikh_Chu_Peleato_Eckstein_2011}. While newer results have extended ADMM for nonconvex problems \cite{Hong_Luo_Razaviyayn_2015}, there is little work addressing the case when $g$ is defined implicitly through $\calD_\sigma$. To the best of our knowledge, the only existing convergence analysis, to date, is the one by Sreehari et al. \cite{Sreehari_Venkatakrishnan_Wohlberg_2015} for the case when $\calD_\sigma$ is a symmetric smoothing filter \cite{Milanfar_2013b,Chan_Zickler_Lu_2015}. However, for general $\calD_\sigma$ the convergence is not known.
\item Original Prior. Since $\calD_\sigma$ is an off-the-shelf image denoising algorithm, it is unclear what prior $g$ does it correspond to. In \cite{Chan_2016}, Chan addresses this question by explicitly deriving the original prior $g$ when $\calD_\sigma$ is a symmetric smoothing filter \cite{Chan_2016}. In this case, the author shows that $g$ is a modified graph Laplacian prior, with better restoration performance compared to the conventional graph Laplacian \cite{Milanfar_2013a}. However, beyond symmetric smoothing filters it becomes unclear if we can find the corresponding $g$.
\item Implementation. The usage of Plug-and-Play ADMM has been reported in a few scattered occasions, with some work in electron tomography \cite{Sreehari_Venkatakrishnan_Wohlberg_2015}, compressive sensing \cite{Dar_Bruckstein_Elad_2015}, and some very recent applications in Poisson recovery \cite{Rond_Giryes_Elad_2015} and super-resolution \cite{Brifman_Romano_Elad_2016}. However, the common challenge underpinning these applications is whether one can obtain a fast solver for the inversion step in \eref{eq:ADMM2,x}. This has not been a problem for conventional ADMM, because in many cases we can use another variable splitting strategy to replace $\vv = \vx$ in \eref{eq:MAP equivalent 2}, e.g., using $\vv = \mB\vx$ when $g(\vx) = \|\mB\vx\|_1$ \cite{Afonso_Bioucas-Dias_Figueiredo_2010}.
\end{enumerate}

\subsection{Related Works}
Plug-and-Play ADMM was first reported in 2013. Around the same period of time there is an independent series of studies using denoisers for  approximate message passing (AMP) algorithms \cite{Metzler_Maleki_Baraniuk_2014,Metzler_Maleki_Baraniuk_2015,Tan_Ma_Baron_2014,Tan_Ma_Baron_2015}. The idea was to replace the shrinkage step of the standard AMP algorithm with any off-the-shelf algorithm in the class of ``proper denoisers'' -- denoisers which ensure that the noise variance is sufficiently suppressed. (See Section \ref{sec:convergence analysis} for more discussions.) However, this type of denoise-AMP algorithms rely heavily on the Gaussian statistics of the random measurement matrix $\mA$ in a specific forward model $f(\vx) = \|\mA\vx-\vy\|^2$. Thus, if $f(\vx)$ departs from quadratic or if $\mA$ is not random, then the behavior of the denoise-AMP becomes unclear.

Using denoisers as building blocks of an image restoration algorithm can be traced back further, e.g., wavelet denoiser for signal deconvolution \cite{Neelamani_Choi_Baraniuk_2004}. Of particular relevance to Plug-and-Play ADMM is the work of Danielyan et al. \cite{Danielyan_Katkovnik_Egiazarian_2012}, where they proposed a variational method for deblurring using BM3D as a prior. The idea was later extended by Zhang et al. to other restoration problems \cite{Zhang_Zhao_Gao_2014}. However, these algorithms are customized for the specific denoiser BM3D. In contrast, the proposed Plug-and-Play ADMM supports any denoiser satisfying appropriate assumptions. Another difference is that when BM3D is used in \cite{Danielyan_Katkovnik_Egiazarian_2012} and \cite{Zhang_Zhao_Gao_2014}, the grouping of the image patches are fixed throughout the iterations. Plug-and-Play ADMM allows re-calculation of the grouping at every iteration. In this aspect, the Plug-and-Play ADMM is more general than these algorithms.

A large number of denoisers we use nowadays are patch-based denoising algorithms. All these methods can be considered as variations in the class of universal denoisers \cite{Weissman_Ordentlich_Seroussi_2005,Sivaramakrishnan_Weissman_2009} which are asymptotically optimal and do not assume external knowledge of the latent image (e.g., prior distribution). Asymptotic optimality of patch-based denoisers has been recognized empirically by Levin et al. \cite{Levin_Nadler_2011,  Levin_Nadler_Durand_2012}, who showed that non-local means \cite{Buades_Coll_2005_Journal} approaches the MMSE estimate as the number of patches grows to infinity. Recently, Ma et al. \cite{Ma_Zhu_Baron_2016} made attempts to integrate universal denoisers with approximate message passing algorithms.

\subsection{Contributions}
The objective of this paper is to address the first and the third issue mentioned in Section~\ref{sec:problems}. The contributions of this paper are as follows:

First, we modify the original Plug-and-Play ADMM by incorporating a continuation scheme. We show that the new algorithm is guaranteed to converge for a broader class of denoisers known as the \emph{bounded denoisers}. Bounded denoisers are asymptotically invariant in the sense that the denoiser approaches an identity operator as the denoising parameter vanishes. Bounded denoisers are weaker than the non-expansive denoisers presented in \cite{Sreehari_Venkatakrishnan_Wohlberg_2015}. However, for weaker denoisers we should also expect a weaker form of convergence. We prove that the new Plug-and-Play ADMM has a \emph{fixed point} convergence, which complements the global convergence results presented in \cite{Sreehari_Venkatakrishnan_Wohlberg_2015}.

Second, we discuss fast implementation techniques for image super-resolution and single photon imaging problems. For the super-resolution problem, conventional ADMM requires multiple variable splits or an inner conjugate gradient solver to solve the subproblem. We propose a polyphase decomposition based method which gives us closed-form solutions. For the single photon imaging problem, existing ADMM algorithm are limited to explicit priors such as total variation. We demonstrate how Plug-and-Play ADMM can be used and we present a fast implementation by exploiting the separable feature of the problem.

The rest of the paper is organized as follows. We first discuss the Plug-and-Play ADMM algorithm and the convergence properties in Section~\ref{sec:convergence}. We then discuss the applications in Section~\ref{sec:applications}. Experimental results are presented in Section~\ref{sec:experiment}.

\section{Plug-and-Play ADMM and Convergence}
\label{sec:convergence}
In this section we present the proposed Plug-and-Play ADMM and discuss its convergence property. Throughout this paper, we assume that the unknown image $\vx$ is bounded in an interval $[x_{\min}, \; x_{\max}]$ where the upper and lower limits can be obtained from experiment or from prior knowledge. Thus, without loss of generality we assume $\vx \in [0,1]^n$.

\subsection{Plug-and-Play ADMM}
The proposed Plug-and-Play ADMM algorithm is a modification of the conventional ADMM algorithm in \eref{eq:ADMM2,x}-\eref{eq:ADMM2,u}. Instead of choosing a constant $\rho$, we increase $\rho$ by $\rho_{k+1} = \gamma_k \rho_k$ for $\gamma_k \ge 1$. In optimization literature, this is known as a continuation scheme \cite{Ng_2002} and has been used in various problems, e.g., \cite{Harmany_Marcia_Willet_2011,Wang_Yang_Yin_2008}. Incorporating this idea into the ADMM algorithm, we obtain the following iteration:
\begin{align}
\vx^{(k+1)} &= \argmin{\vx} \; f(\vx) + (\rho_k/2)\|\vx - (\vv^{(k)}-\vu^{(k)})\|^2 \label{eq:pp admm x}\\
\vv^{(k+1)} &= \calD_{\sigma_k}(\vx^{(k+1)}+\vu^{(k)}) \label{eq:pp admm v}\\
\vu^{(k+1)} &= \vu^{(k)} + (\vx^{(k+1)}-\vv^{(k+1)})\label{eq:pp admm u}\\
\rho_{k+1}  &= \gamma_k \rho_k, \label{eq:pp admm rho}
\end{align}
where $\calD_{\sigma_k}$ is a denoising algorithm (called a ``denoiser'' for short), and $\sigma_k \bydef \sqrt{\lambda/\rho_k}$ is a parameter controlling the strength of the denoiser.

There are different options in setting the update rule for $\rho_k$. In this paper we present two options. The first one is a \emph{monotone update} rule which defines
\begin{equation}
\rho_{k+1} = \gamma\rho_k, \quad\quad \mbox{for all} \;\; k
\end{equation}
for a constant $\gamma > 1$. The second option is an \emph{adaptive update} rule by considering the relative residue:
\begin{align}
\Delta_{k+1} &\bydef \frac{1}{\sqrt{n}}\Big(\|\vx^{(k+1)}-\vx^{(k)}\|_2 + \|\vv^{(k+1)}-\vv^{(k)}\|_2 \notag \\
&\quad\quad + \|\vu^{(k+1)}-\vu^{(k)}\|_2 \Big).
\label{eq:Delta}
\end{align}
For any $\eta \in [0,\,1)$ and let $\gamma > 1$ be a constant, we conditionally update $\rho_k$ according to the followings:
\begin{itemize}
\item If $\Delta_{k+1} \ge \eta \Delta_{k}$, then $\rho_{k+1} = \gamma\rho_k$.
\item If $\Delta_{k+1} < \eta \Delta_{k}$, then $\rho_{k+1} = \rho_k$.
\end{itemize}
The adaptive update scheme is inspired from \cite{Goldstein_ODonoghue_Setzer_2012}, which was originally used to accelerate ADMM algorithms for convex problems. It is different from the residual balancing technique commonly used in ADMM, e.g., \cite{Boyd_Parikh_Chu_Peleato_Eckstein_2011}, as $\Delta_{k+1}$ sums of all primal and dual residues instead of treating them individually. Our experience shows that the proposed scheme is more robust than residual balancing because the denoiser could potentially generate nonlinear effects to the residuals. Algorithm~\ref{alg:algorithm1} shows the overall Plug-and-Play ADMM.

\begin{algorithm}[t]
\caption{Plug-and-Play ADMM}
\begin{algorithmic}
\STATE Input: $\rho_0$, $\lambda$, $\eta < 1$, $\gamma > 1$.
\WHILE{Not Converge}
    \STATE $\vx^{(k+1)} = \argmin{\vx} \; f(\vx) + (\rho_k/2)\|\vx - (\vv^{(k)}-\vu^{(k)})\|^2$
    \STATE $\vv^{(k+1)} = \calD_{\sigma_k}(\vx^{(k+1)}+\vu^{(k)})$, where $\sigma_k = \sqrt{\lambda/\rho_k}$
    \STATE $\vu^{(k+1)} = \vu^{(k)} + (\vx^{(k+1)}-\vv^{(k+1)})$
    \IF{$\Delta_{k+1} \ge \eta \Delta_{k}$}
    \STATE $\rho_{k+1} = \gamma \rho_k$
    \ELSE
    \STATE $\rho_{k+1} = \rho_k$
    \ENDIF
    \STATE $k = k+1$.
\ENDWHILE
\end{algorithmic}
\label{alg:algorithm1}
\end{algorithm}

\begin{remark}[Comparison with \cite{Venkatakrishnan_Bouman_Wohlberg_2013}]
In the original Plug-and-Play ADMM by Venkatakrishnan et al. \cite{Venkatakrishnan_Bouman_Wohlberg_2013}, the update scheme is $\rho_k = \rho$ for some constant $\rho$. This is valid when the denoiser $\calD_{\sigma}$ is non-expansive and has symmetric gradient. However, for general denoisers which could be expansive, the update scheme for $\rho_k$ becomes crucial to the convergence. (See discussion about non-expansiveness in Section~\ref{sec:convergence type}.)
\end{remark}

\begin{remark}[Role of $\sigma_k$]
Many denoising algorithms nowadays such as BM3D and non-local means require one major parameter \footnote{A denoising algorithm often involves many other ``internal'' parameters. However, as these internal parameters do not have direct interaction with the ADMM algorithm, in this paper we keep all internal parameters in their default settings to simplify the analysis.}, typically an estimate of the noise level, to control the strength of the denoiser. In our algorithm, the parameter $\sigma_k$ in \eref{eq:pp admm v} is reminiscent to the noise level. However, unlike BM3D and non-local means where $\sigma_k$ is directly linked to the standard deviation of the i.i.d. Gaussian noise, in Plug-and-Play ADMM we treat $\sigma_k$ simply as a tunable knob to control the amount of denoising because the residue $(\vv-\vvtilde^{(k)})$ at the $k$th iterate is not exactly Gaussian. The adoption of the Gaussian denoiser $\calD_{\sigma_k}$ is purely based on the formal equivalence between \eref{eq:ADMM2,v rewrite} and a Gaussian denoising problem.
\end{remark}

\begin{remark}[Role of $\lambda$]
In this paper, we assume that the parameter $\lambda$ is pre-defined by the user and is fixed. Its role is similar to the regularization parameter in the conventional ADMM problem. Tuning $\lambda$ can be done using external tools such as cross validation \cite{Nguyen_Milanfar_Golub_2001} or SURE \cite{Ramani_Blu_Unser_2008}.
\end{remark}

\subsection{Global and Fixed Point Convergence}
\label{sec:convergence type}
Before we discuss the convergence behavior, we clarify two types of convergence.

We refer to the type of convergence in the conventional ADMM as \emph{global convergence}, i.e., convergence in primal residue, primal objective and dual variables. To ensure global convergence, one sufficient condition is that $g$ is convex, proper and closed \cite{Boyd_Parikh_Chu_Peleato_Eckstein_2011}. For Plug-and-Play ADMM, a sufficient condition is that $\calD_{\sigma}$ has symmetric gradient and is non-expansive \cite{Sreehari_Venkatakrishnan_Wohlberg_2015}. In this case, $g$ exists due to a proximal mapping theorem of Moreau \cite{Moreau_1965}. However, proving non-expansive denoisers could be difficult as it requires
$$\|\calD_{\sigma}(\vx) - \calD_{\sigma}(\vy)\|^2 \le \kappa \|\vx - \vy\|^2$$
for \emph{any} $\vx$ and $\vy$, with $\kappa \le 1$. Even for algorithms as simple as non-local means, one can verify numerically that there exists pairs $(\vx,\vy)$ that would cause $\kappa > 1$. In the Appendix we demonstrate a counter example.

Since $\calD_{\sigma}$ can be arbitrary and we do not even know the existence of $g$, we consider fixed point convergence instead. Fixed point convergence guarantees that a nonlinear algorithm can enter into a steady state asymptotically. In nonlinear dynamical systems, these limit points are referred to as the stable-fixed-points. For any initial guess lying in a region called the basin of attraction the algorithm will converge \cite{Wiggins_1990}. For Plug-and-Play ADMM, we conjecture that fixed point convergence is the best we can ask for unless further assumptions are made on the denoisers.

\subsection{Convergence Analysis of Plug-and-Play ADMM}
\label{sec:convergence analysis}
We define the class of bounded denoisers.

\begin{definition}{(Bounded Denoiser).}
A \emph{bounded denoiser} with a parameter $\sigma$ is a function $\calD_{\sigma}: \R^n \rightarrow \R^n$ such that for any input $\vx \in \R^n$,
\begin{equation}
\|\calD_{\sigma}(\vx)-\vx\|^2/n \le \sigma^2 C,
\label{eq:bounded denoiser}
\end{equation}
for some universal constant $C$ independent of $n$ and $\sigma$.
\end{definition}

Bounded denoisers are asymptotically invariant in the sense that it ensures $\calD_{\sigma} \rightarrow \calI$ (i.e., the identity operator) as $\sigma \rightarrow 0$. It is a weak condition which we expect most denoisers to have. The asymptotic invariant property of a bounded denoiser prevents trivial mappings from being considered, e.g., $\calD_{\sigma}(\vx) = 0$ for all $\vx$.

\begin{remark}
It would be useful to compare a bounded denoiser with a ``proper denoiser'' defined in \cite{Metzler_Maleki_Baraniuk_2014}. A proper denoiser $\widetilde{\calD}_{\sigma}$ is a mapping that denoises a noisy input $\vx+\sigma\vepsilon$ with the property that
\begin{equation}
\E\left[\|\widetilde{\calD}_{\sigma}(\vx+\sigma\vepsilon)-\vx\|^2/n\right] \le \kappa \sigma^2,
\label{eq:proper denoiser}
\end{equation}
for any $\kappa < 1$, where $\vepsilon \sim \calN(0,\mI)$ is the i.i.d. Gaussian noise. Note that in \eref{eq:proper denoiser}, we require the input to be a deterministic signal $\vx$ plus an i.i.d. Gaussian noise. Moreover, the parameter must match with the noise level. In contrast, a bounded denoiser can take any input and any parameter.
\end{remark}

Besides the conditions on $\calD_{\sigma}$ we also assume that the negative log-likelihood function $f$ has bounded gradients:
\begin{assumption}
\label{assumption:bound gradient}
We assume that $f: [0,1]^n \rightarrow \R$ has bounded gradients. That is, for any $\vx \in [0,1]^n$, there exists $L < \infty$ such that $\|\nabla f(\vx)\|_2/\sqrt{n} \le L$.
\end{assumption}

\begin{example}
Let $f(\vx) = \|\mA\vx - \vy\|_2^2$ for $\mA \in \R^{n \times n}$ with eigenvalues bounded between 0 and 1. The gradient of $f$ is $\nabla f(\vx) = 2\mA^T(\mA\vx - \vy)$ and
\begin{align*}
\| \nabla f(\vx) \|_{2}/\sqrt{n}
&\le 2 \lambda_{\max}(\mA)^2(\|\vx\|_2+\|\vy\|_2)/\sqrt{n}.
\end{align*}
\end{example}

The main convergence result of this paper is as follows.

\begin{theorem}{(Fixed Point Convergence of Plug-and-Play ADMM).}
\label{thm:main}
Under Assumption~\ref{assumption:bound gradient} and for any bounded denoiser $\calD_{\sigma}$, the iterates of the Plug-and-Play ADMM defined in Algorithm~\ref{alg:algorithm1} demonstrates a fixed-point convergence. That is, there exists $(\vx^*,\vv^*,\vu^*)$ such that $\|\vx^{(k)}-\vx^*\|_2 \rightarrow 0$, $\|\vv^{(k)}-\vv^*\|_2 \rightarrow 0$ and $\|\vu^{(k)}-\vu^*\|_2 \rightarrow 0$ as $k \rightarrow \infty$.
\end{theorem}

\begin{proof}
See Appendix B.
\end{proof}

Intuitively, what Theorem~\ref{thm:main} states is that as $k \rightarrow \infty$, the continuation scheme forces $\rho_k \rightarrow \infty$. Therefore, the inversion in \eref{eq:pp admm x} and the denoising in \eref{eq:pp admm v} have reducing influence as $\rho_k$ grows. Hence, the algorithm converges to a fixed point. Theorem~\ref{thm:main} also ensures that $\vx^{(k)} \rightarrow \vv^{(k)}$ which is an important property of the original Plug-and-Play ADMM algorithm \cite{Sreehari_Venkatakrishnan_Wohlberg_2015}. The convergence of $\vx^{(k)} \rightarrow \vv^{(k)}$ holds because $\vu^{(k+1)} = \vu^{(k)} + (\vx^{(k+1)}-\vv^{(k+1)})$ converges. In practice, experimentally we observe that if the algorithm is terminated early to reduce the runtime, then $\vv^{(k)}$ tends to provide a slightly better solution.

\subsection{Stopping Criteria}
Since we are seeking for fixed point convergence, a natural stopping criteria is to determine if $\|\vx^{(k+1)}-\vx^{(k)}\|_2$, $\|\vv^{(k+1)}-\vv^{(k)}\|_2$ and $\|\vu^{(k+1)}-\vu^{(k)}\|_2$ are sufficiently small. Following the definition of $\Delta_{k+1}$ in \eref{eq:Delta}, we choose to terminate the iteration when
\begin{align}
\Delta_{k+1} &\bydef \frac{1}{\sqrt{n}}\Big(\|\vx^{(k+1)}-\vx^{(k)}\|_2 + \|\vv^{(k+1)}-\vv^{(k)}\|_2 \notag \\
&\quad\quad + \|\vu^{(k+1)}-\vu^{(k)}\|_2 \Big) \le \mathtt{tol}
\end{align}
for some tolerance level $\mathtt{tol}$. Alternatively, we can also terminate the algorithm when
\begin{align*}
\max\Big\{\epsilon_1, \epsilon_2, \epsilon_3\Big\} \le \mathtt{tol}/3,
\end{align*}
where $\epsilon_1 = \|\vx^{(k+1)}-\vx^{(k)}\|_2/\sqrt{n}$, $\epsilon_2 = \|\vv^{(k+1)}-\vv^{(k)}\|_2/\sqrt{n}$ and
$\epsilon_3 = \|\vu^{(k+1)}-\vu^{(k)}\|_2/\sqrt{n}$.

In practice, the tolerance level does not need to be extremely small in order to achieve good reconstruction quality. In fact, for many images we have tested, setting $\mathtt{tol} \approx 10^{-3}$ is often sufficient. \fref{fig:stopping} provides a justification. In this experiment, we tested an image super-resolution problem for 10 testing images (See Configuration 3 in Section~\ref{sec:experiment super} for details). It can be observed that the PSNR becomes steady when $\mathtt{tol}$ drops below $10^{-3}$. Moreover, size of the image does not seem to be an influencing factor. Smaller images such as \texttt{Cameraman256}, \texttt{House256} and \texttt{Peppers256} shows similar characteristics as bigger images. The more influencing factor is the combination of the update ratio $\gamma$ and the initial value $\rho_0$. However, unless $\gamma$ is close to 1 and $\rho_0$ is extremely small (which does not yield good reconstruction anyway), our experience is that setting $\mathtt{tol}$ at $10^{-3}$ is usually valid for $\gamma \in (1,2)$ and $\rho_0 \in (10^{-5},10^{-2})$.

\begin{figure}[h]
\centering
\includegraphics[width=0.9\linewidth]{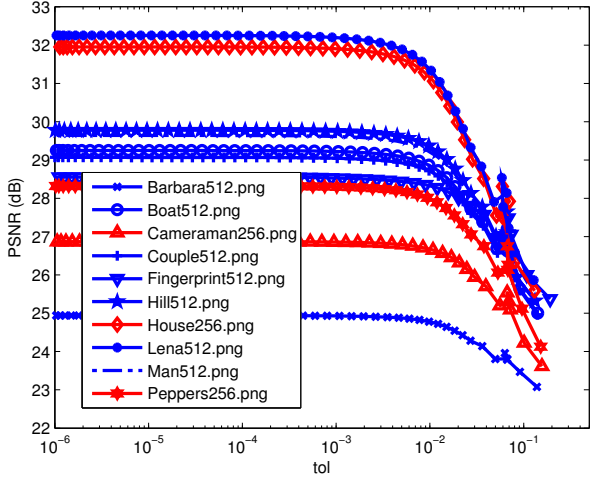}
\caption{Stopping criteria. The PSNR drops as the tolerance level increases. However, regardless of the size of the images, the PSNR becomes steady when $\mathtt{tol} \approx 10^{-3}$.}
\label{fig:stopping}
\end{figure}

\subsection{Initial Parameter $\rho_0$}
The choice of the initial parameter $\rho_0$ requires some tuning but is typically good for $\rho_0 \in (10^{-5},10^{-2})$. \fref{fig:rho0,super} shows the behavior of the algorithm for different values of $\rho_0$, ranging from $10^{0}$ to $10^{-4}$. We compare the original Plug-and-Play ADMM (i.e., with constant $\rho_k = \rho_0$, the red lines), monotone update rule (i.e., $\rho_{k+1} = \gamma \rho_k$, the blue lines), and the adaptive update rule (the black lines). We make two observations regarding the difference between the proposed algorithm and the original Plug-and-Play ADMM \cite{Sreehari_Venkatakrishnan_Wohlberg_2015}:

\begin{densitemize}
\item \textbf{Stability}: The original Plug-and-Play ADMM \cite{Sreehari_Venkatakrishnan_Wohlberg_2015} requires a highly precise $\rho_0$. For example, in \fref{fig:rho0,super} the best PSNR is achieved when $\rho_0 = 1$; When $\rho_0$ is less than $10^{-2}$, the PSNR becomes very poor. The proposed algorithm works for a much wider range of $\rho_0$.
\item \textbf{Final PSNR}: The proposed Plug-and-Play ADMM is a generalization of the original Plug-and-Play ADMM. The added degrees of freedom are the new parameters $(\rho_0,\gamma,\eta)$. The original Plug-and-Play ADMM is a special case when $\gamma = 1$. Therefore, for optimally tuned parameters, the proposed Plug-and-Play ADMM is always better than or equal to the original Plug-and-Play ADMM. This is verified in \fref{fig:rho0,super}, which shows that the best PSNR is attained by the proposed method.
\end{densitemize}

\begin{figure}[h]
\centering
\includegraphics[width=0.9\linewidth]{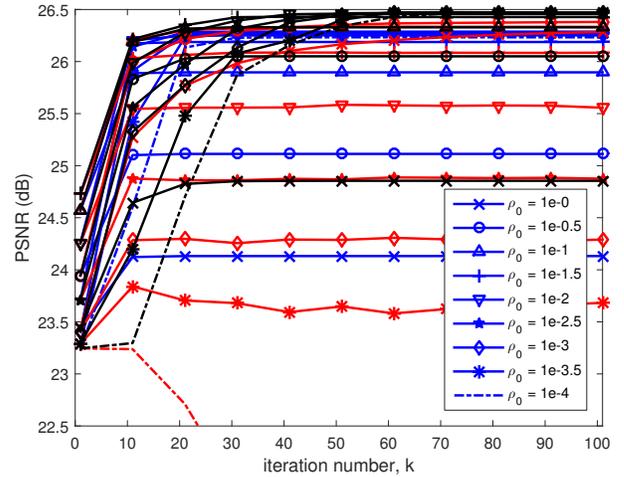}
\caption{Influence of $\rho_0$. Red curves represent the original method in \cite{Sreehari_Venkatakrishnan_Wohlberg_2015}; Blue curves represent the monotone update rule; Black curves represent the adaptive update rule. Note the diversified behavior of the red curve, which implies that a precise $\rho_0$ is required. The blue and black curves are more robust.}
\label{fig:rho0,super}
\end{figure}

\subsection{Initial Guesses}
\begin{figure}[t]
\centering
\includegraphics[width=0.9\linewidth]{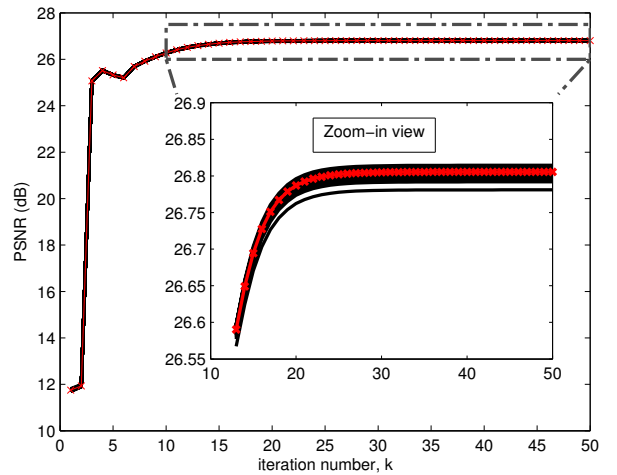}
\caption{Influence of the initial point $\vx^{(0)}$. We start the algorithm with 100 different initial guesses $\vx^{(0)}$ where each is a uniformly random vector drawn from $[0,1]^n$. Over these 100 random realizations we plot the average (red line). Note the small fluctuation of the PSNR at the limit.}
\label{fig:initial point}
\vspace{-2ex}
\end{figure}

The initial guesses $\vx^{(0)}$, $\vv^{(0)}$ and $\vu^{(0)}$ have less impact to the final PSNR. This can be seen from \fref{fig:initial point}. In this experiment, we randomly draw 100 initial guesses $\vx^{(0)}$ from a uniform distribution in $[0,1]^n$. The auxiliary variable is set as $\vv^{(0)} = \vx^{(0)}$, and the Lagrange multiplier $\vu^{(0)}$ is 0. As shown in \fref{fig:initial point}, the initial guesses do not cause significant difference in term of PSNR at the limit. The standard deviation at the limit is 0.0059 dB, implying that with 99.7\% probability (3 standard deviations) the PSNR will stay within $\pm 0.0176$ dB from its average.

\section{Applications}
\label{sec:applications}
As we discussed in the introduction, Plug-and-Play ADMM algorithm has a wide range of applications. However, in order to enable the denoising step, Plug-and-Play ADMM uses a specific variable splitting strategy. The challenge it brings, therefore, is whether we can solve the subsequent subproblems efficiently. The purpose of this section is to address this issue by presenting two applications where fast implementation can be achieved.

\subsection{Application 1: Image Super-resolution}
Image super-resolution can be described by a linear forward model with two operations: an anti-aliasing filter and a subsampling process. The function $f(\vx)$ is quadratic in the form
\begin{equation}
f(\vx) = \|\mS\mH\vx - \vy\|^2,
\label{eq:f gaussian}
\end{equation}
where the matrix $\mH \in \R^{n \times n}$ is a circulant matrix representing the convolution for the anti-aliasing filter. The matrix $\mS \in \R^{m \times n}$ is a binary sampling matrix, where the rows are subsets of the identity matrix. By defining $\mG \bydef \mS\mH$ we recognize that when substituting \eref{eq:f gaussian} into \eref{eq:ADMM2,x}, the $f$-subproblem becomes (we dropped the iteration number $k$ to simplify notation)
\begin{equation}
\vxhat = \argmin{\vx\in \R^n} \;\; \|\mG\vx-\vy\|^2 +\frac{\rho}{2} \|\vx - \vxtilde\|^2.
\end{equation}
Consequently, the solution is the pseudo-inverse
\begin{equation}
\vxhat = (\mG^T\mG + \rho\mI)^{-1}(\mG^T\vy + \rho\vxtilde).
\label{eq:pseudo inverse}
\end{equation}
For special cases of $\mH$ and $\mS$, \eref{eq:pseudo inverse} has known efficient implementation as follows.

\begin{example}[Non-blind deblurring \cite{Chan_Khoshabeh_Gibson_2011,Afonso_Bioucas-Dias_Figueiredo_2010,Wang_Yang_Yin_2008}]
Non-blind deblurring is a special case when $\mS = \mI$. In this case, since $\mH$ is circulant which is diagonalizable by the discrete Fourier transform matrices, \eref{eq:pseudo inverse} can be efficiently implemented by
\begin{equation}
\vxhat = \calF^{-1}\left\{ \frac{ \overline{\calF(h)}\calF(\vy) + \rho\calF(\vxtilde) }{ |\calF(h)|^2 + \rho} \right\},
\end{equation}
where $\calF(\cdot)$ is the Fourier transform operator, $h$ is the finite impulse response filter representing the blur kernel, $\overline{(\cdot)}$ is the complex conjugate, and the multiplication/division are element-wise operations.
\end{example}

\begin{example}[Interpolation \cite{Dahl_Hansen_Jensen_2010,Liu_Chan_Nguyen_2015,Garcia_2010,Zhou_Chen_Ren_2009}]
Image interpolation is a special case when $\mH = \mI$. In this case, since $\mS^T\mS$ is a diagonal matrix with binary entries, \eref{eq:pseudo inverse} can be efficiently implemented using an element-wise division:
\begin{equation}
\vxhat = (\mS^T\vy + \rho\vxtilde) ./ (\vs + \rho),
\end{equation}
where $\vs = \diag{\mS^T\mS}$.
\end{example}

\subsection{Polyphase Implementation for Image Super-Resolution}
When $\mG = \mS\mH$, solving the $f$-subproblem becomes non-trivial because $\mH^T\mS^T\mS\mH$ is neither diagonal nor diagonalizable by the Fourier transform. In literature, the two most common approaches are to introduce multi-variable split to bypass \eref{eq:pseudo inverse} (e.g., \cite{Afonso_Bioucas-Dias_Figueiredo_2010,Almeida_Figueiredo_2013}) or use an inner conjugate gradient to solve \eref{eq:pseudo inverse} (e.g., \cite{Brifman_Romano_Elad_2016}). However, multi-variable splitting requires additional Lagrange multipliers and internal parameters. It also generally leads to slower convergence than single-variable split. Inner conjugate gradient is computationally expensive as it requires an iterative solver. In what follows, we show that when $\mS$ is the standard $K$-fold downsampler (i.e., sub-sample the spatial grid uniformly with a factor $K$ along horizontal and vertical directions), and when $\mH$ is a circular convolution, it is possible to derive a closed-form solution \footnote{We assume the boundaries are circularly padded. In case of other types boundary conditions or unknown boundary conditions, we can pre-process the image by padding the boundaries circularly. Then, after the super-resolution algorithm we crop the center region. The alternative approach is to consider multiple variable split as discussed in  \cite{Almeida_Figueiredo_2013}. }.

Our closed form solution begins by considering the Sherman-Morrison-Woodbury identity, which allows us to rewrite \eref{eq:pseudo inverse} as
\begin{equation}
\vxhat = \rho^{-1}\vb - \rho^{-1}\mG^T(\rho\mI + \mG\mG^T)^{-1}\mG\vb,
\label{eq:f super solution 2}
\end{equation}
where $\vb \bydef \mG^T\vy + \rho\vxtilde$. Note that if $\mS \in \R^{m \times n}$ and $\mH \in \R^{n \times n}$ with $m < n$, then \eref{eq:f super solution 2} only involves a $m \times m$ inverse, which is smaller than the $n \times n$ inverse in \eref{eq:pseudo inverse}.

The more critical step is the following observation. We note that the matrix $\mG\mG^T$ is given by
\begin{equation*}
\mG\mG^T = \mS\mH\mH^T\mS^T.
\end{equation*}
Since $\mS$ is a $K$-fold downsampling operator, $\mS^T$ is a $K$-fold upsampling operator. Defining $\mHtilde = \mH\mH^T$, which can be implemented as a convolution between the blur kernel $h$ and its time-reversal, we observe that $\mS\mHtilde\mS^T$ is a ``upsample-filter-downsample'' sequence. This idea is illustrated in \fref{fig:combine gtg}.

We next study the polyphase decomposition \cite{Vaidyanathan_1992} of \fref{fig:combine gtg}. Polyphase decomposition allows us to write
\begin{equation}
\widetilde{H}(z) = \sum_{k=0}^{K-1}z^{-k}\widetilde{H}_k(z^K),
\label{eq:polyphase}
\end{equation}
where $\widetilde{H}(z)$ is the $z$-transform representation of the blur matrix $\mHtilde = \mH\mH^T$, and $\widetilde{H}_k(z^K)$ is the $k$th polyphase component of $\widetilde{H}(z)$. Illustrating \eref{eq:polyphase} using a block diagram, we show in \fref{fig:polyphase gtg} the decomposed structure of \fref{fig:combine gtg}. Then, using Noble identity \cite{Vaidyanathan_1992}, the block diagram on the left hand side of \fref{fig:polyphase gtg} becomes the one shown on the right hand side. Since for any $k>1$, placing a delay $z^{-k}$ between an upsampling and a downsampling operator leads to a zero, the overall system simplifies to a finite impulse response filter $\Htilde_0(z)$, which can be pre-computed.

We summarize this by the following proposition.
\begin{proposition}
\label{prop:polyphase}
The operation of $\mS\mH\mH^T\mS^T$ is equivalent to applying a finite impulse response filter $\Htilde_0(z)$, which is the 0th polyphase component of the filter $\mH\mH^T$.
\end{proposition}

\begin{figure*}[!]
\centering
\scalebox{0.75}{
\begin{pspicture}(0,-1)(15,0.5)
\psline{-}(0,0)(0.5,0)
\rput(1,0){\upsampling{$K$}}
\psline{-}(1.5,0)(2,0)
\rput(2.75,0){\shortfilter{$\overline{H(z)}$}}
\psline{-}(3.5,0)(4,0)
\rput(4.75,0){\shortfilter{$H(z)$}}
\psline{-}(5.5,0)(6,0)
\rput(6.5,0){\downsampling{$K$}}
\psline{-}(7,0)(7.5,0)
\psframe[linewidth=1pt,framearc=0.3,linestyle=dashed](0.35,-0.75)(3.65,0.75)
\psframe[linewidth=1pt,framearc=0.3,linestyle=dashed](3.85,-0.75)(7.15,0.75)
\rput(2,-1){$\mG^T$}
\rput(5.5,-1){$\mG$}
\rput(8.25,0){$\equiv$}
\psline{-}(9,0)(9.5,0)
\rput(10,0){\upsampling{$K$}}
\psline{-}(10.5,0)(11,0)
\rput(11.75,0){\shortfilter{$\widetilde{H}(z)$}}
\psline{-}(12.5,0)(13,0)
\rput(13.5,0){\downsampling{$K$}}
\psline{-}(14,0)(14.5,0)
\end{pspicture}}
\caption{[Left] Block diagram of the operation $\mG\mG^T$. [Right] The equivalent system, where $\widetilde{H}(z) \bydef \overline{H(z)}H(z)$.}
\label{fig:combine gtg}
\end{figure*}
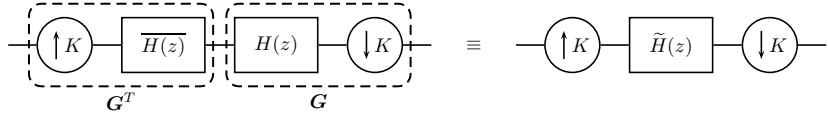

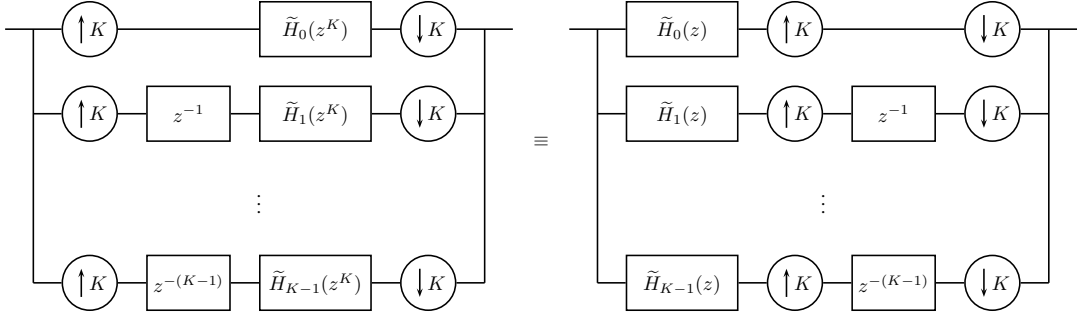
\begin{figure*}[!]
\centering
\scalebox{0.75}{
\begin{pspicture}(0,-5)(18,1)
\rput(0,0){\mybranchA{$\Htilde_0(z^K)$}{$K$}}
\rput(0,-1.5){\mybranch{$\Htilde_1(z^K)$}{$K$}{$z^{-1}$}}
\rput(4,-3){$\vdots$}
\rput(0,-4.5){\mybranch{$\Htilde_{K-1}(z^K)$}{$K$}{$z^{-(K-1)}$}}
\psline{-}(0,0)(0,-4.5)
\psline{-}(8,0)(8,-4.5)
\rput(10,0){\mybranchC{$\Htilde_0(z)$}{$K$}}
\rput(10,-1.5){\mybranchB{$\Htilde_1(z)$}{$K$}{$z^{-1}$}}
\rput(14,-3){$\vdots$}
\rput(10,-4.5){\mybranchB{$\Htilde_{K-1}(z)$}{$K$}{$z^{-(K-1)}$}}
\psline{-}(10,0)(10,-4.5)
\psline{-}(18,0)(18,-4.5)
\rput(9,-2){$\equiv$}
\end{pspicture}
}
\caption{[Left] Polyphase decomposition of $\widetilde{H}(z)$. [Right] Equivalent representation. }
\label{fig:polyphase gtg}
\end{figure*}

To implement the 0th polyphase component, we observe that it can be done by downsampling the convolved filter $\mHtilde = \mH\mH^T$. This leads to the procedure illustrated in Algorithm~\ref{alg:polyphase}.

The implication of Proposition~\ref{prop:polyphase} is that since $\mG\mG^T$ is equivalent to a finite impulse response filter $\widetilde{h}_0$,  \eref{eq:f super solution 2} can be implemented in closed-form using the Fourier transform:
\begin{equation}
\vx = \rho^{-1}\vb - \rho^{-1}\mG^T\left( \calF^{-1}\left\{\frac{\calF(\mG\vb)}{| \calF(\widetilde{h}_0)|^2 + \rho}\right\}\right),
\end{equation}
where we recall that $\vb = \mG^T\vy + \rho\vxtilde$.

\begin{algorithm}[t]
\caption{Compute the $0$th polyphase component.}
\label{alg:polyphase}
\begin{algorithmic}
\STATE Input: $h$: the blur kernel, and $K$: downsampling factor
\STATE Let $\widetilde{h} = \calF^{-1}( \calF(h) \overline{\calF(h)})$ be the convolved filter.
\STATE Output: $\widetilde{h}_0 = (\downarrow_K) (\widetilde{h})$.
\end{algorithmic}
\end{algorithm}

\begin{figure}[t]
\centering
\includegraphics[width=0.8\linewidth]{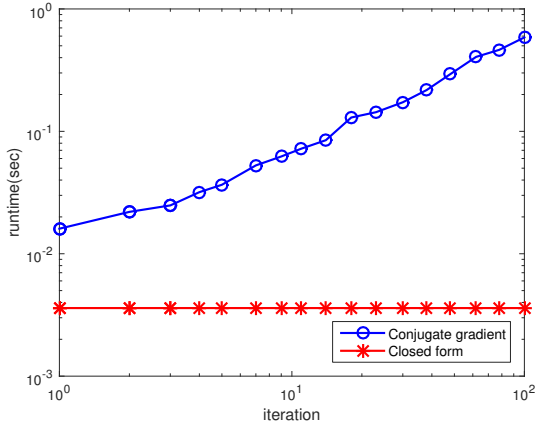}
\caption{Runtime of conjugate gradient for solving a $\vx$-subproblem. Note the non-iterative nature of the closed-form solution.}
\label{fig:runtime}
\end{figure}

The effectiveness of the proposed closed-form solution can be seen from \fref{fig:runtime}. In this figure, we compare with a brute force conjugate gradient method presented in \cite{Brifman_Romano_Elad_2016}. When $\mH$ satisfies periodic boundary conditions, the closed-form solution is \emph{exact}. If the boundaries are not periodic, alternative solutions can be considered, e.g., \cite{Matakos_Ramani_Fessler_2013}.

\subsection{Application 2: Single Photon Imaging}
The second application is a single photon imaging problem using quanta image sensors (QIS) \cite{Fossum_2011}. Using ADMM for QIS was previously reported in \cite{Chan_Lu_2014,Elgendy_Chan_2016}. Here, we show how the Plug-and-Play ADMM can be used for the problem.

QIS is a spatial oversampling device. A QIS is composed to many tiny single photon detectors called jots. In each unit space, $K$ jots are used to acquire light corresponding to a pixel in the usual sense (e.g., a pixel in a CMOS sensor). Therefore, for an image of $n$ pixels, a total number of $nK$ jots are required. By assuming homogeneous distribution of the light within each pixel, we consider a simplified QIS imaging model which relates the underlying image $\vx \in \R^n$ and the actual photon arrival rate at the jots $\vs \in \R^{nK}$ as
\begin{equation*}
\vs = \alpha\mG\vx,
\end{equation*}
\noindent where the matrix $\mG \in \R^{nK \times n}$ is
\begin{equation}
\mG = \frac{1}{K} \;
\begin{bmatrix}
\vone_{K \times 1} & \vzero_{K \times 1} & \ldots & \vzero_{K \times 1}\\
\vzero_{K \times 1}& \vone_{K \times 1}  & \ldots & \vzero_{K \times 1}\\
\vdots & \vdots & \ddots & \vdots \\
\vzero_{K \times 1}& \vzero_{K \times 1} & \ldots & \vone_{K \times 1}
\end{bmatrix},
\label{eq:G single photon}
\end{equation}
and $\alpha$ is a sensor gain. Given $\vs$, the photons arriving at the sensors follow a Poisson distribution with a rate given by $\vs$. Let $Z_i$ be the random variable denoting the number of photons at jot $i$, we have
\begin{equation}
p(z_i) = \frac{s_i^{-z_i}e^{-s_i}}{z_i!}, \quad i = 1,\ldots,nK.
\end{equation}
The final QIS output, $Y_i$, is a binary bit resulted from truncating $Z_i$ using a threshold $q$. That is,
\begin{equation*}
Y_i =
\begin{cases}
1, &\quad \mbox{if } Z_i \ge q,\\
0, &\quad \mbox{if } Z_i < q.
\end{cases}
\end{equation*}
When $q = 1$, the probability of observing $Y_i = y_i$ given $s_i$ is
\begin{equation*}
p( y_i \,|\, s_i) =
\begin{cases}
e^{-s_i},   &\quad \mbox{if } y_i = 0,\\
1-e^{-s_i}, &\quad \mbox{if } y_i = 1.
\end{cases}
\end{equation*}

The recovery goal is to estimate $\vx$ from the observed binary bits $\vy$. Taking the negative log and summing over all pixels, the function $f$ is defined as
\begin{align}
f(\vx)
&\bydef p(\vy \,|\, \vs) = \sum_{i=1}^{nK} -\log \; p( y_i \,|\, s_i) \notag \\
&= \sum_{i=1}^{nK} -\log\left( (1-y_i)e^{-s_i} + y_i\left(1-e^{-s_i}\right) \right) \notag \\
&= \sum_{j=1}^n - K_j^0\log(e^{-\frac{\alpha x_j}{K}}) - K_j^1\log(1-e^{-\frac{\alpha x_j}{K}}) ,
\label{eq:single photon f}
\end{align}
where $K_j^1 = \sum_{i=1}^K y_{(j-1)K+i}$ is the number of ones in the $j$th unit pixel, and $K_j^0 = \sum_{i=1}^K (1-y_{(j-1)K+i})$ is the number of zeros in the $j$th unit pixel. (Note that for any $j$, $K^1_j+K^0_j = K$.) Consequently, substituting \eref{eq:single photon f} into \eref{eq:pp admm x} yields the $f$-subproblem
\begin{equation}
\min_{\vx} \sum_{j=1}^n - K_j^0\log(e^{-\frac{\alpha x_j}{K}}) - K_j^1\log(1-e^{-\frac{\alpha x_j}{K}}) + \frac{\rho}{2}(x_j-\widetilde{x}_j)^2.
\label{eq:single photon f subproblem}
\end{equation}
Since this optimization is separable, we can solve each individual variable $x_j$ independently. Thus, for every $j$, we solve a single-variable optimization by taking derivative with respect to $x_j$ and setting to zero, yielding
\begin{equation*}
Ke^{-\frac{\alpha x_j}{K}}(\alpha+\rho(x_j-\widetilde{x}_j)) = \alpha K^0_j + \rho K(x_j-\widetilde{x}_j),
\end{equation*}
which is a one-dimensional root finding problem. By constructing an offline lookup table in terms of $K_0$, $\rho$ and $\widetilde{x}_j$, we can solve \eref{eq:single photon f subproblem} efficiently.

\section{Experimental Results}
\label{sec:experiment}
In this section we present the experimental results. For consistency we use BM3D in all experiments, although other bounded denoisers will also work. We shall not compare Plug-and-Play ADMM using different denoisers as it is not the focus of the paper.

\subsection{Image Super-Resolution}
\label{sec:experiment super}
We consider a set of 10 standard test images for this experiment as shown in \fref{fig:example}. All images are gray-scaled, with sizes between $256 \times 256$ and $512 \times 512$. Four sets of experimental configurations are studied, and are shown in Table~\ref{table:configration}.

\begin{figure}[h]
\centering
\includegraphics[width=0.9\linewidth]{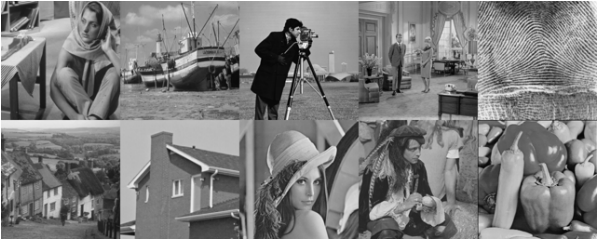}
\caption{10 testing images for the experiment.}
\label{fig:example}
\end{figure}

\begin{table}[h]
\caption{Configurations and parameters.}
\label{table:configration}
\begin{tabular}{cl}
Config & Description \\
\hline
1 & $K = 2$, $\mH = $ bicubic, Noise = 0\\
2 & $K = 4$, $\mH = $ bicubic, Noise = 0\\
  & $\rho_0 = 1.3\times10^{-5}$, $\gamma = 2.5$, $\lambda = 10^{-5}$\\
  \hline
3 & $K = 2$, $\mH = $ Gaussian of std 1, Noise = 5/255\\
4 & $K = 4$, $\mH = $ Gaussian of std 1, Noise = 5/255\\
  & $\rho_0 = 1\times10^{-5}$, $\gamma = 1.2$, $\lambda = 10^{-4}$\\
\hline
\end{tabular}
\end{table}

\begin{table*}[!]
\small
\centering
\caption{Image Super Resolution Results. When noise is present, the PSNR values are averaged over 5 random realizations of the noise pattern. Gray color rows represent external database methods.}
\label{table:SR result}
\begin{tabular}{c|cccccccccccc}
& \multicolumn{10}{c}{Images} &  & \\
\hline
        & 1         & 2         & 3         & 4         & 5         & 6         & 7         & 8         & 9         & 10        &  Dataset    & Avg STD \\
Size    & $512^2$   & $512^2$   & $256^2$   & $512^2$   & $512^2$   & $512^2$   & $256^2$   & $512^2$   & $512^2$   & $256^2$   &  Avg      & per image\\
\hline
\multicolumn{13}{c}{Factor: $\times 2$; Anti-aliasing Filter: Bicubic; Noise: 0} \\
\Xhline{3\arrayrulewidth}
\rowcolor{Gray}
DCNN\cite{Dong_Loy_He_2014}   	&	\multicolumn{1}{!{\vrule width 0.8pt}c!}{25.71} 	&	31.83	&	28.79	&	31.13	&	32.73	&	 \textbf{32.60}     	&	35.11	&	36.34	&	 \textbf{33.10}     	&	33.05	&	32.04	& -- \\
\rowcolor{Gray}
SR \cite{Yang_Wright_Huang_2008}     	&	\multicolumn{1}{!{\vrule width 0.8pt}c!}{\textbf{25.87}}     	&	31.51	&	27.92	&	30.94	&	33.02	&	32.46	&	34.79	&	36.14	&	32.80	&	32.67	&	31.81	 &--\\
\rowcolor{Gray}
SPSR \cite{Peleg_Elad_2014}     	&	\multicolumn{1}{!{\vrule width 0.8pt}c!}{25.71}	&	31.49	&	27.85	&	31.00	&	33.30	&	32.35	&	34.37	&	36.18	&	32.67	&	32.66	&	31.76	 &--\\
\rowcolor{White}
TSE \cite{Huang_Singh_Ahuia_2015}     	&	\multicolumn{1}{!{\vrule width 0.8pt}c!}{25.66}	&	31.64	&	28.17	&	31.01	&	32.88	&	32.45	&	34.78	&	36.22	&	32.88	&	 \textbf{33.29}     	&	31.90	 &--\\
\rowcolor{White}
GPR \cite{He_Siu_2011}     	&	\multicolumn{1}{!{\vrule width 0.8pt}c!}{24.99}	&	29.99	&	26.44	&	29.50	&	30.36	&	31.33	&	33.08	&	34.19	&	31.09	&	30.98	&	30.20	 &--\\
\rowcolor{White}
Ours - M      	&	\multicolumn{1}{!{\vrule width 0.8pt}c!}{25.87}	&	 \textbf{31.85}     	&	 \textbf{28.81}     	&	 \textbf{31.42}     	&	 \textbf{33.63}     	&	32.56	&	 \textbf{35.30}     	&	 \textbf{36.45}     	&	32.86	&	33.06	&	 \textbf{32.18}     	& --\\
\rowcolor{White}
Ours - A      	&	\multicolumn{1}{!{\vrule width 0.8pt}c!}{25.74}	&	31.68	&	28.44	&	31.27	&	33.49	&	32.39	&	35.20	&	36.02	&	32.65	&	32.65	&	31.95	& --\\
\hline
\multicolumn{13}{c}{Factor: $\times 4$; Anti-aliasing Filter: Bicubic; Noise: 0} \\
\Xhline{3\arrayrulewidth}
\rowcolor{Gray}
DCNN \cite{Dong_Loy_He_2014}     	&	\multicolumn{1}{!{\vrule width 0.8pt}c!}{23.93}	&	26.84	&	23.76	&	26.08	&	24.18	&	28.32	&	29.48	&	30.45	&	28.17	&	27.88	&	26.91	&	 --\\									
\rowcolor{Gray}		
SR \cite{Yang_Wright_Huang_2008}    	&	\multicolumn{1}{!{\vrule width 0.8pt}c!}{23.91}	&	26.39	&	23.43	&	26.02	&	24.44	&	28.17	&	29.15	&	30.19	&	27.94	&	27.42	&	26.71	&	 --\\		
\rowcolor{Gray}									
SPSR \cite{Peleg_Elad_2014}     	&	\multicolumn{1}{!{\vrule width 0.8pt}c!}{23.90}	&	26.49	&	23.42	&	26.02	&	25.11	&	28.14	&	29.22	&	30.24	&	27.85	&	27.62	&	26.80	&	 --\\				\rowcolor{White}
TSE \cite{Huang_Singh_Ahuia_2015}     	&	\multicolumn{1}{!{\vrule width 0.8pt}c!}{23.90}	&	26.62	&	23.83	&	26.10	&	24.70	&	28.06	&	30.03	&	30.29	&	28.03	&	27.97	&	26.95	&	 --\\			
\rowcolor{White}								
GPR \cite{He_Siu_2011}     	&	\multicolumn{1}{!{\vrule width 0.8pt}c!}{23.55}	&	25.47	&	22.54	&	25.27	&	22.33	&	27.79	&	27.61	&	28.74	&	26.76	&	25.79	&	25.58	&	 --\\						
\rowcolor{White}					
Ours - M      	&	 \multicolumn{1}{!{\vrule width 0.8pt}c!}{\textbf{23.99}}     	&	26.87	&	23.82	&	 \textbf{26.32}     	&	25.54	&	 \textbf{28.35}     	&	30.16	&	 \textbf{30.74}     	&	 \textbf{28.17}     	&	 \textbf{28.26}     	&	27.22	&	 --\\		
\rowcolor{White}									
Ours - A      	&	\multicolumn{1}{!{\vrule width 0.8pt}c!}{23.99}	&	 \textbf{26.87}     	&	\textbf{23.83}     	&	26.33	&	 \textbf{25.58}     	&	28.29	&	  \textbf{30.48}   	&	30.62	&	28.12	&	28.22	&	  \textbf{27.23}  	&	 --\\											
\hline
\multicolumn{13}{c}{Factor: $\times 2$; Anti-aliasing Filter: Gaussian $9 \times 9$, $\sigma = 1$; Noise: $5/255$} \\
\Xhline{3\arrayrulewidth}
\rowcolor{Gray}
DCNN \cite{Dong_Loy_He_2014}     	&	\multicolumn{1}{!{\vrule width 0.8pt}c!}{23.61}	&	26.30	&	23.75	&	26.21	&	23.79	&	27.50	&	27.84	&	28.15	&	27.05	&	26.07	&	26.03	&	 0.0219\\					
\rowcolor{Gray}						
SR \cite{Yang_Wright_Huang_2008}     	&	\multicolumn{1}{!{\vrule width 0.8pt}c!}{23.61}	&	26.25	&	23.71	&	26.15	&	23.80	&	27.41	&	27.71	&	28.07	&	26.99	&	26.10	&	25.98	&	 0.0221\\			
\rowcolor{Gray}								
SPSR \cite{Peleg_Elad_2014}     	&	\multicolumn{1}{!{\vrule width 0.8pt}c!}{23.75}	&	26.57	&	23.88	&	26.47	&	23.91	&	27.80	&	28.19	&	28.58	&	27.33	&	26.39	&	26.29	&	 0.0208\\			\rowcolor{White}			
TSE \cite{Huang_Singh_Ahuia_2015}     	&	\multicolumn{1}{!{\vrule width 0.8pt}c!}{23.57}	&	26.22	&	23.65	&	26.12	&	23.79	&	27.34	&	27.55	&	28.00	&	26.93	&	26.11	&	25.93	&	 0.0238\\		
\rowcolor{White}									
GPR \cite{He_Siu_2011}     	&	\multicolumn{1}{!{\vrule width 0.8pt}c!}{23.82}	&	26.81	&	23.91	&	26.63	&	24.05	&	28.38	&	29.16	&	29.54	&	27.78	&	26.76	&	26.68	&	 0.0170\\			
\rowcolor{White}								
Ours - M      	&	 \multicolumn{1}{!{\vrule width 0.8pt}c!}{\textbf{24.64}}     	&	\textbf{29.41}     	&	 \textbf{26.73}     	&	 \textbf{29.22}     	&	 \textbf{28.82}     	&	 \textbf{29.82}     	&	 \textbf{32.65}     	&	 \textbf{32.76}     	&	 \textbf{29.66}     	&	 \textbf{30.10}     	&	 \textbf{29.38}   	&	 0.0267\\	
\rowcolor{White}										
Ours - M*     	&	\multicolumn{1}{!{\vrule width 0.8pt}c!}{24.01}	&	27.09	&	24.17	&	27.00	&	25.03	&	28.46	&	29.32	&	29.78	&	27.85	&	26.99	&	26.97	&	 0.0178\\											
\hline
\multicolumn{13}{c}{Factor: $\times 4$; Anti-aliasing Filter: Gaussian $9 \times 9$, $\sigma = 1$; Noise: $5/255$} \\
\hline
\rowcolor{Gray}
DCNN \cite{Dong_Loy_He_2014}     	&	\multicolumn{1}{!{\vrule width 0.8pt}c!}{20.72}	&	21.30	&	18.91	&	21.68	&	16.10	&	23.39	&	22.33	&	22.99	&	22.46	&	20.23	&	21.01	&	 0.0232\\						
\rowcolor{Gray}					
SR \cite{Yang_Wright_Huang_2008}     	&	\multicolumn{1}{!{\vrule width 0.8pt}c!}{20.67}	&	21.30	&	18.86	&	21.51	&	16.37	&	23.15	&	22.19	&	22.85	&	22.26	&	20.33	&	20.95	&	 0.0212\\			
\rowcolor{Gray}								
SPSR \cite{Peleg_Elad_2014}     	&	\multicolumn{1}{!{\vrule width 0.8pt}c!}{20.85}	&	21.58	&	19.18	&	21.85	&	16.59	&	23.52	&	22.42	&	23.05	&	22.53	&	20.50	&	21.21	&	 0.0217\\			\rowcolor{White}								
TSE \cite{Huang_Singh_Ahuia_2015}     	&	\multicolumn{1}{!{\vrule width 0.8pt}c!}{20.59}	&	21.24	&	18.80	&	21.49	&	16.40	&	23.14	&	22.21	&	22.78	&	22.21	&	20.30	&	20.92	&	 0.0252\\		
\rowcolor{White}									
GPR \cite{He_Siu_2011}     	&	\multicolumn{1}{!{\vrule width 0.8pt}c!}{21.55}	&	22.68	&	19.90	&	22.77	&	17.70	&	24.57	&	23.51	&	24.37	&	23.63	&	21.35	&	22.20	&	 0.0313\\	
\rowcolor{White}										
Ours - M      	&	 \multicolumn{1}{!{\vrule width 0.8pt}c!}{\textbf{23.62}}     	&	 \textbf{25.75}    	&	 \textbf{23.06}     	&	 \textbf{25.30}     	&	 \textbf{24.48}     	&	 \textbf{27.17}    	&	 \textbf{29.14}     	&	 \textbf{29.42}     	&	 \textbf{26.86}     	&	 \textbf{26.86}     	&	 \textbf{26.17}   	&	 0.0223\\		
\rowcolor{White}									
Ours - M*     	&	\multicolumn{1}{!{\vrule width 0.8pt}c!}{21.21}	&	22.12	&	19.43	&	22.43	&	16.90	&	24.37	&	23.13	&	23.95	&	23.39	&	21.13	&	21.81	&	 0.0253\\											
\hline
\end{tabular}
\end{table*}

\begin{figure*}[!]
\centering
\footnotesize
\begin{tabular}{cccc}
\includegraphics[width=0.2\linewidth]{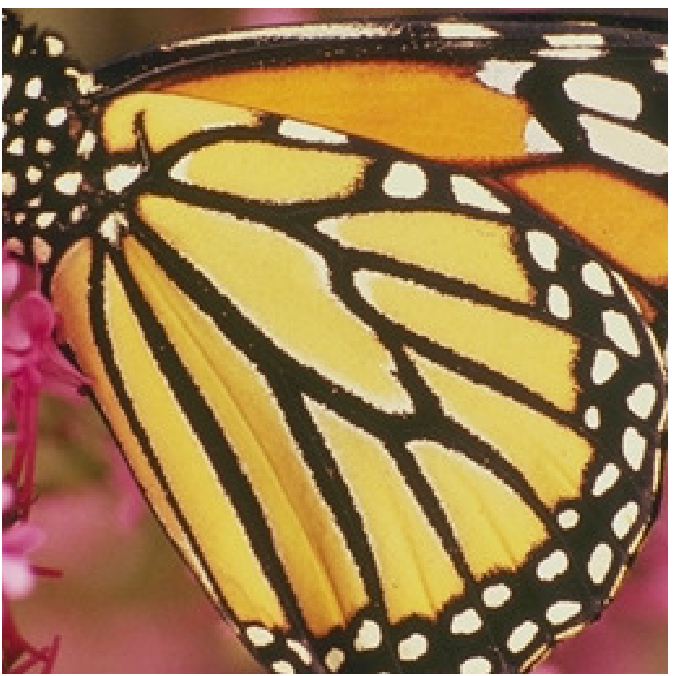}&
\includegraphics[width=0.2\linewidth]{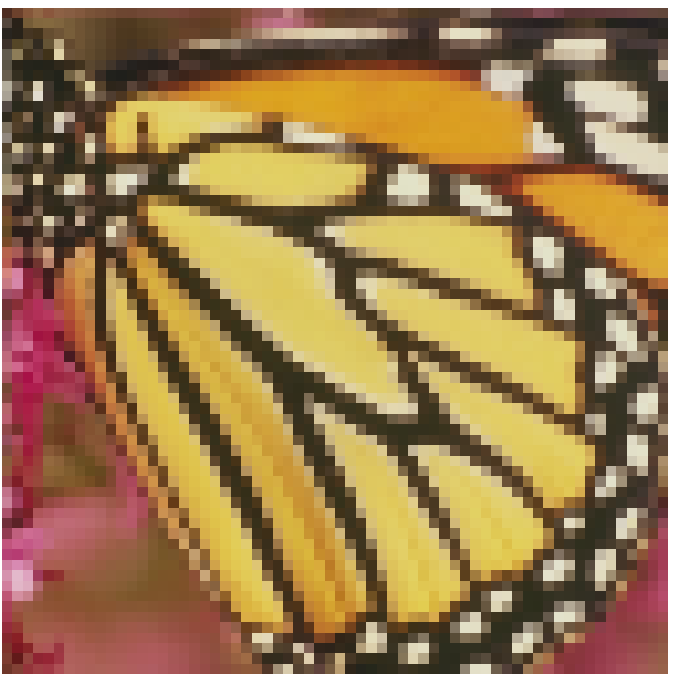}&
\includegraphics[width=0.2\linewidth]{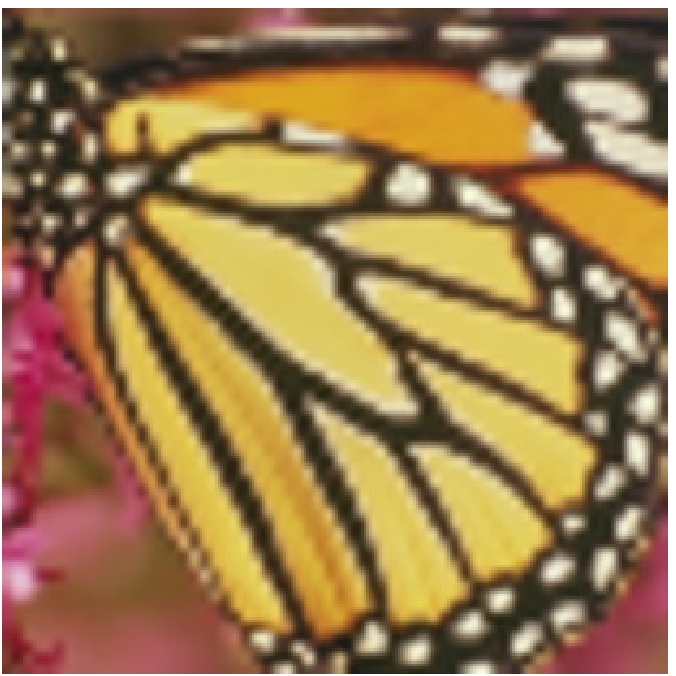}&
\includegraphics[width=0.2\linewidth]{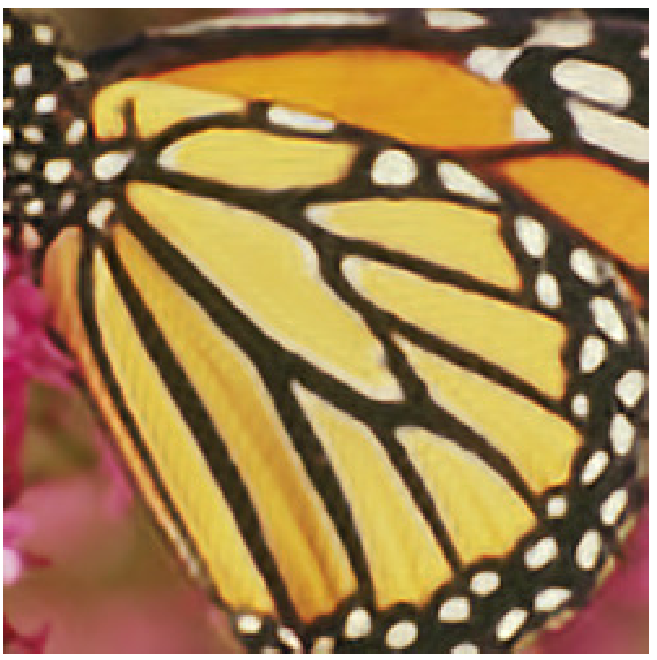}\\
\includegraphics[width=0.2\linewidth]{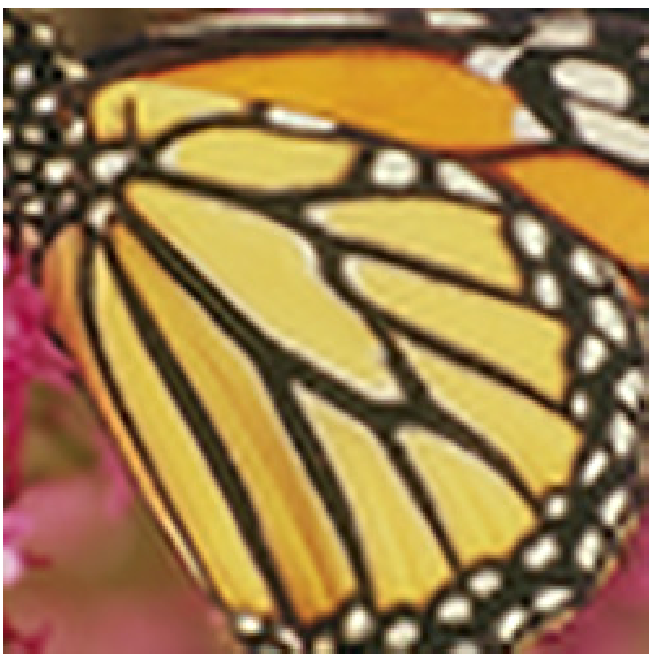}&
\includegraphics[width=0.2\linewidth]{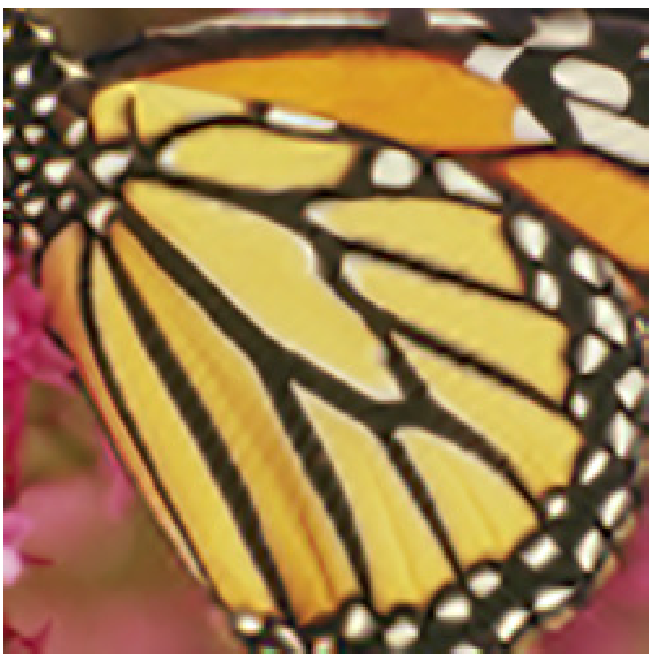}&
\includegraphics[width=0.2\linewidth]{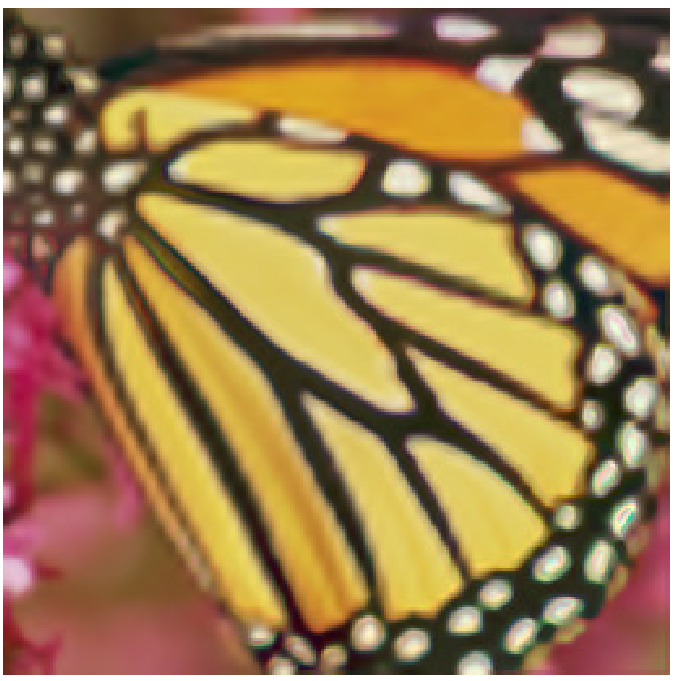}&
\includegraphics[width=0.2\linewidth]{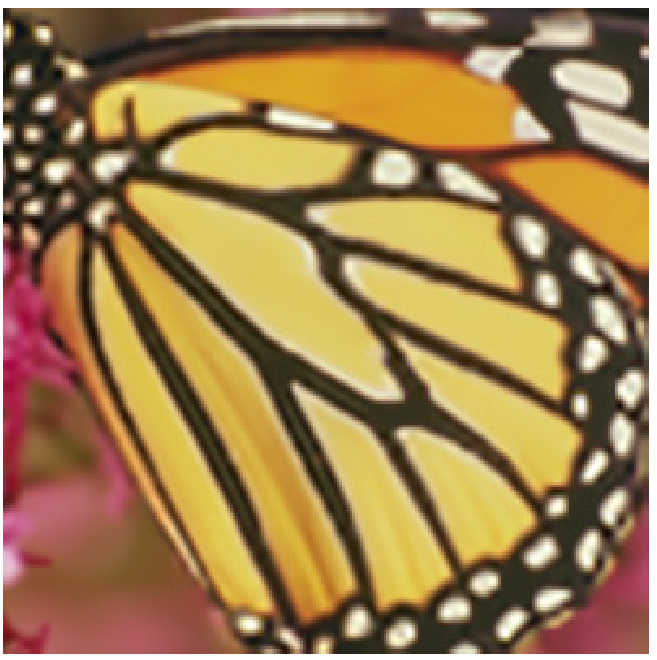}\\
\end{tabular}
\caption{Image Super Resolution Results. [Top](from left to right). Ground truth; The low resolution input; Bicubic interpolation; DCNN \cite{Dong_Loy_He_2014} (24.19dB). [Bottom](from left to right). SPSR \cite{Peleg_Elad_2014} (22.44dB); TSE \cite{Huang_Singh_Ahuia_2015} (22.80dB); GPR \cite{He_Siu_2011} (20.81dB); Ours-M (23.49dB).}
\label{fig:SR images 1}
\end{figure*}

We compared the proposed algorithm with several existing super-resolution algorithms. These methods include the deep convolutional neural network method (DCNN) by Dong et al. \cite{Dong_Loy_He_2014}, the statistical patch-based sparse representation method (SPSR) by Peleg and Elad \cite{Peleg_Elad_2014}, the transformed self-exemplar method (TSE) by Huang et al. \cite{Huang_Singh_Ahuia_2015}, the classical sparse representation method for super-resolution (SR) by Yang et al. \cite{Yang_Wright_Huang_2008}, and the Gaussian process regression method (GPR) by He and Siu \cite{He_Siu_2011}. Among these methods, we note that TSE and GPR are single image methods whereas DCNN, SPSR and SR require training using external databases.

The results of the experiment are shown in Table~\ref{table:SR result}. For the proposed algorithm, we present the two update rules as Ours-M (for monotone update rule), and Ours-A (for adaptive update rule). When noise is present in the simulation, we conduct a Monte-Carlo simulation over 5 random realizations. In this case, the reported PSNR values are the average over the random realizations. The per image standard deviation is reported in the last column of Table~\ref{table:SR result} (if applicable).

For configurations 3 and 4 when we use a Gaussian anti-aliasing filter, we observe that not all existing methods can handle such case as the training part of those algorithms was performed on a bicubic model. Therefore, for fair comparison, we present two versions of the proposed algorithm. The first version Ours-M assumes the correct knowledge about the Gaussian filter, whereas the second version Ours-M* ignores such assumption and use the bicubic model for reconstruction.

From the PSNR results shown in Table~\ref{table:SR result}, we observe very similar performance of the competing methods. For configurations 1 and 2, the proposed algorithm shows the best performance overall, although in some occasions the deep neural network \cite{Dong_Loy_He_2014} is better. For configurations 3 and 4, we observe a significant gap between Ours-M and the competing methods. This is caused by the model mismatch of the competing methods as the implementations provided by the authors only support the bicubic model. For fairness, we consider Ours-M* by pretending that the anti-aliasing filter is bicubic. In this case, Ours-M* still performs better than the others for configuration 3, but slightly worse than GPR \cite{He_Siu_2011} for configuration 4.

For visual comparison we conduct a color image experiment. In this experiment, we simulate a low resolution image by downsampling the color image by a factor 4 using a bicubic anti-aliasing filter. Then, we apply the proposed algorithm to the three color channels individually to recover the image. The result is shown in \fref{fig:SR images 1}. As seen, the proposed method produces better results than SPSR \cite{Peleg_Elad_2014}, TSE \cite{Huang_Singh_Ahuia_2015} and GPR \cite{He_Siu_2011}, with slightly sharper edges and less halo artifacts. We also observe that the deep neural network \cite{Dong_Loy_He_2014} shows better results than that in Table~\ref{table:SR result}. One possibility is that the training data used to train the neural network are natural images that have better correlation to \fref{fig:SR images 1}. However, considering the training-free nature of the proposed Plug-and-Play algorithm, losing to a well-trained neural network is not surprising.

\begin{table*}[!]
\small
\centering
\caption{Single Photon Imaging Results. The PSNR values are averaged over 8 random realizations of the photon arrivals.}
\label{table:single photon}
\begin{tabular}{c|cccccccccccc}
& \multicolumn{10}{c}{Images} &  & \\
\hline
       & 1         & 2         & 3         & 4         & 5         & 6         & 7         & 8         & 9         & 10        &  Dataset      & Avg STD \\
Size   & $512^2$   & $512^2$   & $256^2$   & $512^2$   & $512^2$   & $512^2$   & $256^2$   & $512^2$   & $512^2$   & $256^2$   &  Avg      & per image\\
\hline
\multicolumn{13}{c}{ $K = 4$ } \\
\hline
Yang et al. \cite{Yang_Lu_Sbaiz_2012}  & 14.80 & 14.18 & 14.68 & 14.39 & 14.28 & 14.90 & 14.21 & 14.48 & 14.78 & 14.59 & 14.53 & 0.0157 \\
Chan-Lu \cite{Chan_Lu_2014}      &22.59 & 24.76 & 23.63 & 24.74 & 21.47 & 25.65 & 25.38 & 26.42 & 25.35 & 24.78 & 24.48 & 0.0425\\
Ours-M        &25.99 & 25.72 & 25.58 & 26.03 & 23.54 & 26.60 & 28.15 & 28.17 & 26.17 & 26.10 & 26.20 & 0.0821\\
Ours-A        &\textbf{26.06} & \textbf{25.75} & \textbf{25.64} & \textbf{26.10} & \textbf{23.59} & \textbf{26.66} & \textbf{28.27} & \textbf{28.27 }& \textbf{26.18} & \textbf{26.16} & \textbf{26.27} & 0.0801\\
\hline
\multicolumn{13}{c}{ $K = 6$ } \\
\hline
Yang et al. \cite{Yang_Lu_Sbaiz_2012}  &17.94 & 17.27 & 17.67 & 17.61 & 17.29 & 18.22 & 17.22 & 17.62 & 18.00 & 17.65 & 17.65 & 0.0147\\
Chan-Lu \cite{Chan_Lu_2014}      &23.97 & 26.25 & 25.53 & 26.26 & 24.47 & 26.66 & 26.75 & 27.32 & 26.58 & 26.40 & 26.02 & 0.0339\\
Ours-M        & \textbf{28.34} & \textbf{27.76} & 27.60 & \textbf{27.91} & \textbf{25.66} & \textbf{28.30} & \textbf{30.34} & \textbf{30.06} & \textbf{27.75} & 28.15 & \textbf{28.19} & 0.0451\\
Ours-A        & \textbf{28.34} & 27.72 & \textbf{27.61} & 27.84 & 25.62 & 28.25 & 30.28 & 29.86 & 27.71 & \textbf{28.16} & 28.14 & 0.0472\\
\hline
\multicolumn{13}{c}{ $K = 8$ } \\
\hline
Yang et al. \cite{Yang_Lu_Sbaiz_2012}  & 20.28 & 19.68 & 20.00 & 20.05 & 19.53 & 20.64 & 19.49 & 19.99 & 20.42 & 19.97 & 20.01 & 0.0183\\
Chan-Lu \cite{Chan_Lu_2014}      & 25.14 & 27.09 & 26.56 & 27.24 & 25.75 & 27.55 & 27.17 & 27.89 & 27.49 & 27.07 & 26.90 & 0.0325\\
Ours-M        & \textbf{29.79} & \textbf{29.07} & \textbf{29.14} & \textbf{29.25} & \textbf{27.19} & \textbf{29.55} & \textbf{31.70} & \textbf{31.43} & \textbf{28.99} & \textbf{29.52} & \textbf{29.56} & 0.0527\\
Ours-A        &29.74 & 29.00 & 29.09 & 29.16 & 27.14 & 29.51 & 31.54 & 31.35 & 28.94 & 29.43 & 29.49 &  0.0520\\
\hline
\multicolumn{13}{c}{ $K = 10$ } \\
\hline
Yang et al. \cite{Yang_Lu_Sbaiz_2012}  & 22.14 & 21.60 & 21.89 & 21.98 & 21.32 & 22.51 & 21.32 & 21.86 & 22.35 & 21.83 & 21.88 & 0.0198\\
Chan-Lu \cite{Chan_Lu_2014}      &26.20 & 27.57 & 27.26 & 27.70 & 26.41 & 27.99 & 27.64 & 28.16 & 27.95 & 27.66 & 27.46 & 0.0264\\
Ours-M        & \textbf{30.88} & \textbf{30.19} & \textbf{30.34} & \textbf{30.31} & \textbf{28.29} & \textbf{30.48} & \textbf{32.68} & \textbf{32.29} & \textbf{29.97} & \textbf{30.56} & \textbf{30.60} & 0.0386\\
Ours-A        &30.81 & 30.12 & 30.31 & 30.22 & 28.22 & 30.41 & 32.51 & 32.17 & 29.90 & 30.47 & 30.51 & 0.0397\\
\hline
\end{tabular}
\end{table*}

\begin{figure*}[!]
\centering
\footnotesize
\begin{tabular}{cccc}
\includegraphics[width=0.23\linewidth]{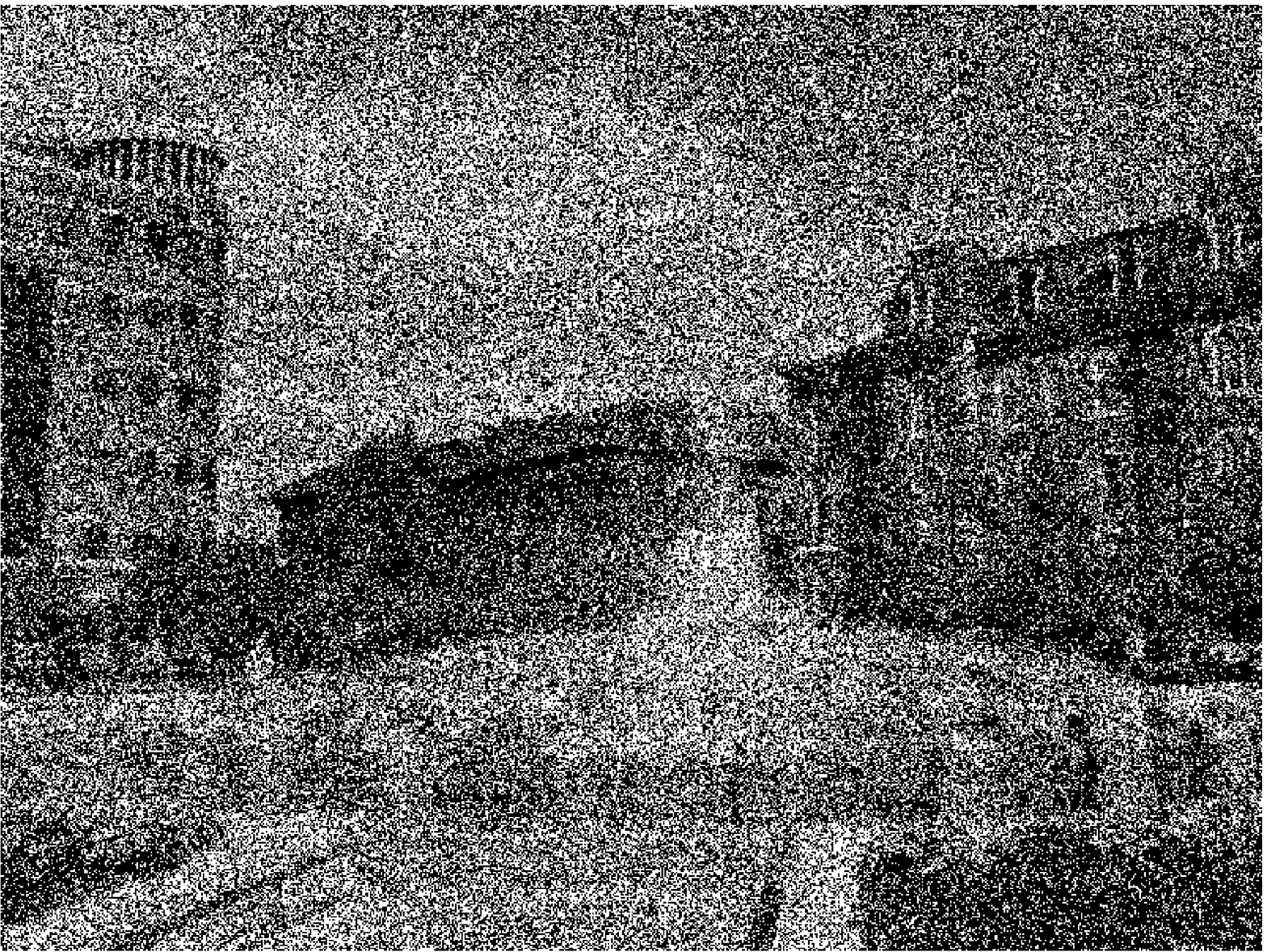}&
\includegraphics[width=0.23\linewidth]{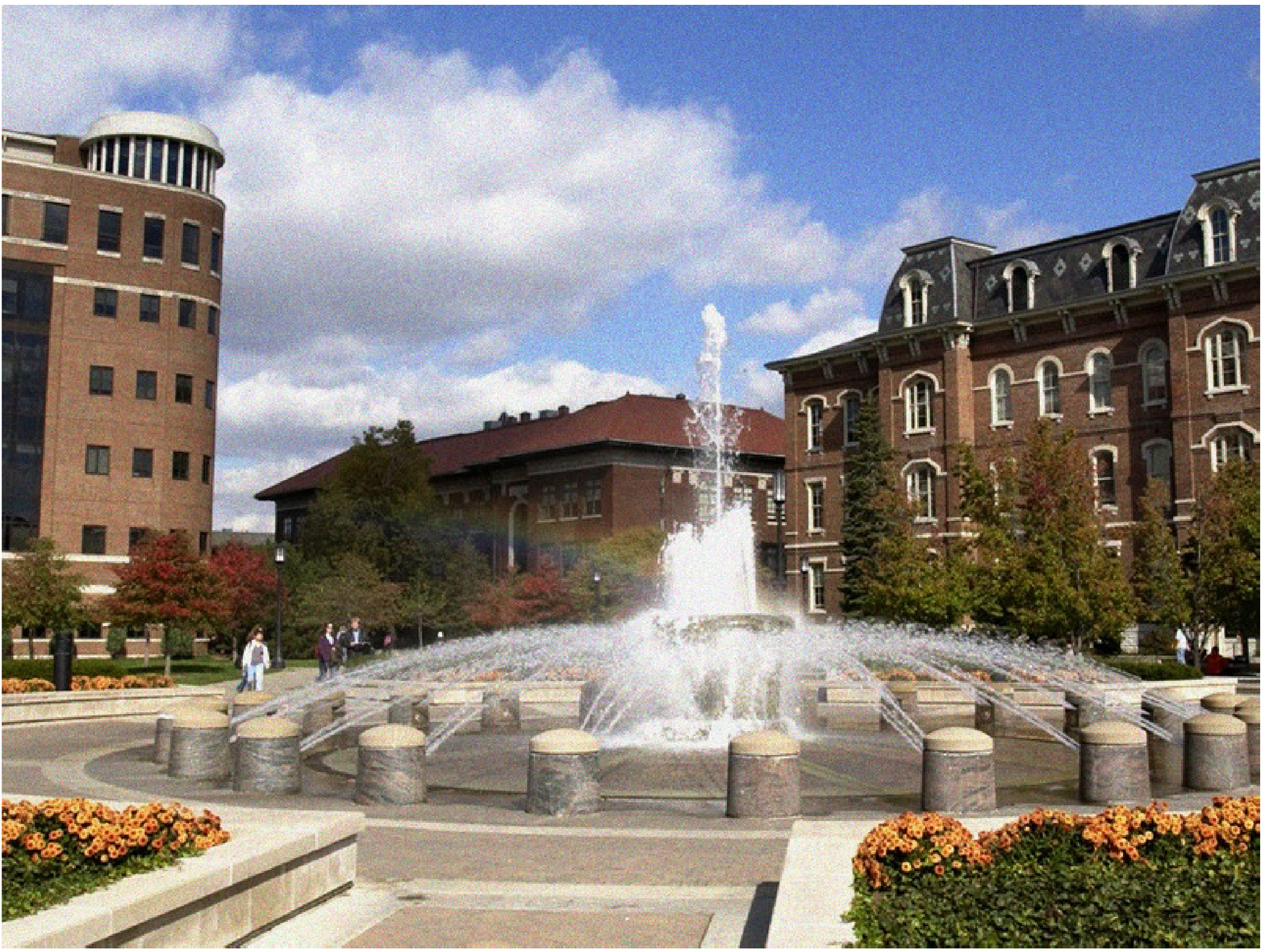}&
\includegraphics[width=0.23\linewidth]{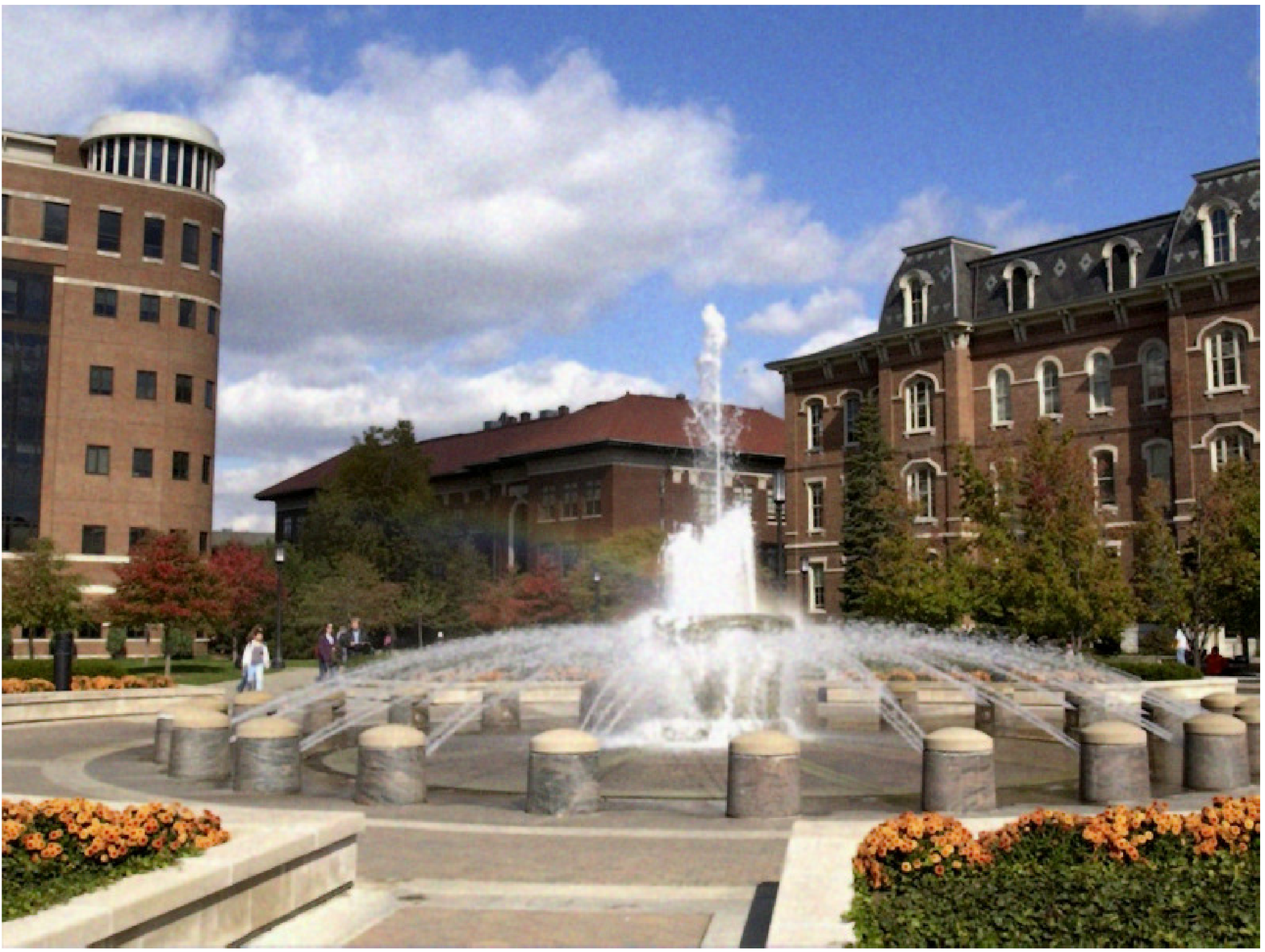}&
\includegraphics[width=0.23\linewidth]{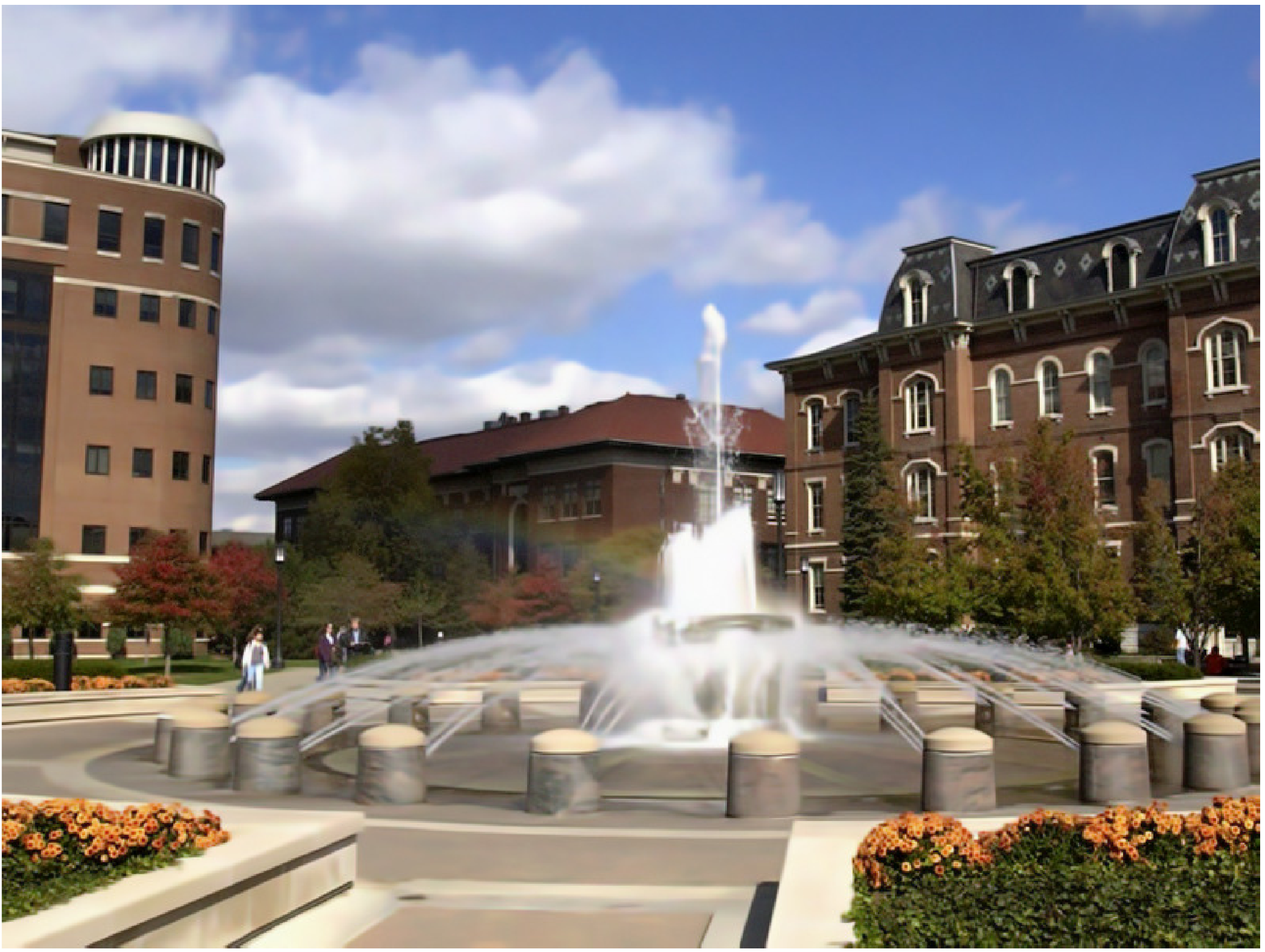}\\
\includegraphics[width=0.23\linewidth]{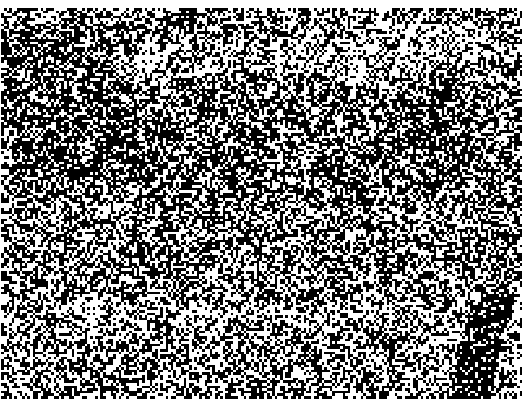}&
\includegraphics[width=0.23\linewidth]{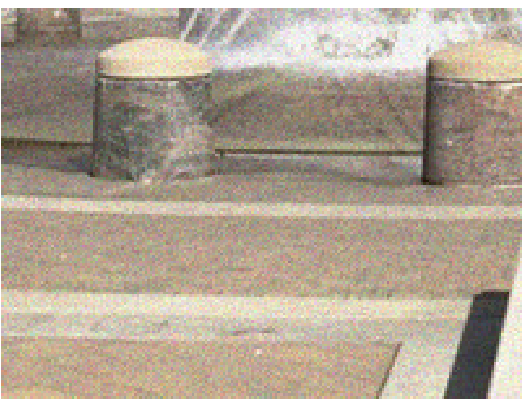}&
\includegraphics[width=0.23\linewidth]{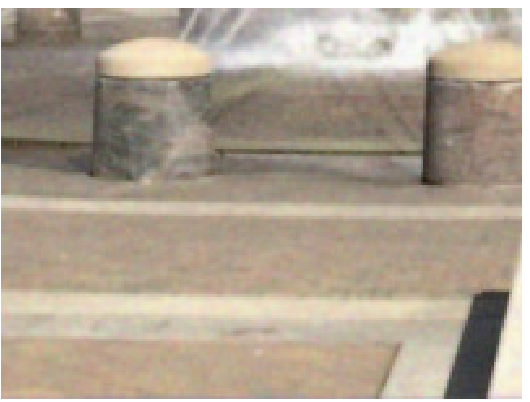}&
\includegraphics[width=0.23\linewidth]{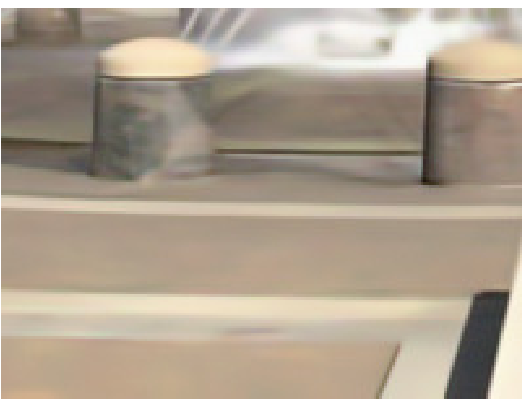}\\
(a) Binary input & (b) Yang et al. \cite{Yang_Lu_Sbaiz_2012} 20.16dB & (c) Chan and Lu \cite{Chan_Lu_2014} 26.74dB & (d) Ours-C 28.81dB
\end{tabular}
\caption{Single photon imaging results. The bottom row is a zoomed-in figure of the top row.}
\label{fig:QIS images}
\end{figure*}

\subsection{Single Photon Imaging}
We next consider the single-photon imaging problem. In this experiment, we consider four sets of experiments for $K = 4, 6, 8, 10$ (along horizontal and vertical directions). The sensor gain is set as $\alpha=K^2$. For comparison, we choose the two existing algorithms. The first one is the maximum likelihood estimation (MLE) method by Yang et al. \cite{Yang_Lu_Sbaiz_2012}. For our specific choice of $\mG$ in \eref{eq:G single photon}, the MLE solution has a closed-form expression. The second method is a total variation method by Chan and Lu \cite{Chan_Lu_2014}. This method utilizes the ADMM algorithm when solving the problem. We are aware of other existing Poisson denoising methods such as \cite{Figueiredo_Bioucas_2010,Rond_Giryes_Elad_2015}. However, none of these methods are directly applicable to the quantized Poisson problem.

Since for this problem the observed binary pattern is a truncated Poisson random variable, we perform a Monte-Carlo simulation by repeating each case for 8 independent trials. We then report the average and the standard deviation of these 8 independent trials. As shown in Table~\ref{table:single photon}, the standard deviation is indeed insignificant compared to the average PSNR. Here, we report the dataset average over the 10 images to ensure sufficient variability of the test. To visually compare the performance, in \fref{fig:QIS images} we show the result of a color image. In this experiment, we process the 3 color channels individually. For each channel, we simulate the photon arrivals by assuming $K = 8$. Then, we reconstruct the image using different algorithms. The result in \fref{fig:QIS images} shows that visually the proposed algorithm produces images with less noise.

\begin{figure*}[!]
\centering
\footnotesize
\begin{tabular}{cccccc}
\includegraphics[width=0.13\linewidth]{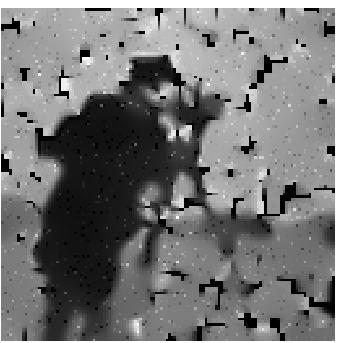}&
\hspace{-1ex}\includegraphics[width=0.13\linewidth]{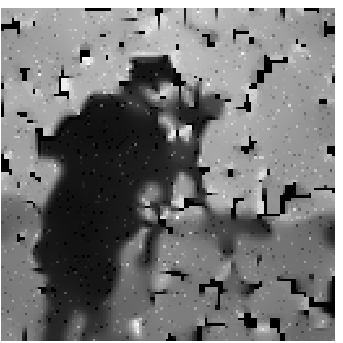}&
\hspace{-1ex}\includegraphics[width=0.13\linewidth]{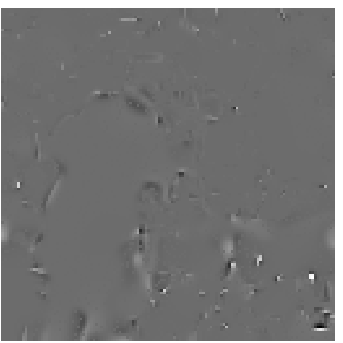}&
\hspace{-1ex}\includegraphics[width=0.13\linewidth]{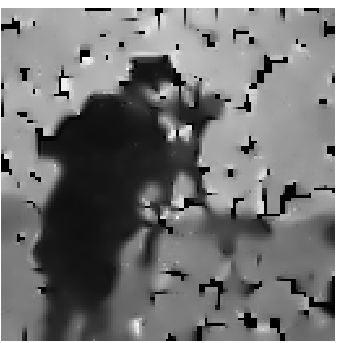}&
\hspace{-1ex}\includegraphics[width=0.13\linewidth]{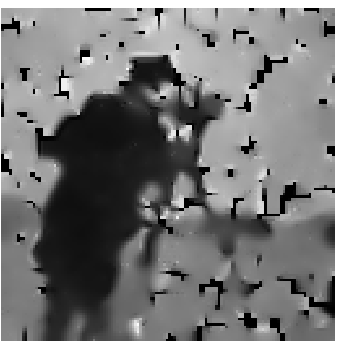}&
\hspace{-1ex}\includegraphics[width=0.13\linewidth]{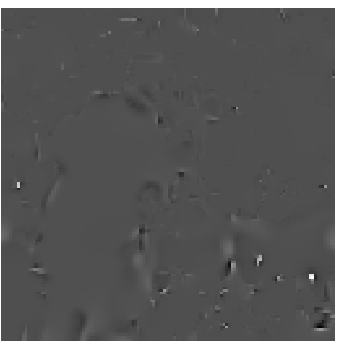}\\
(a) $\vx$ &
\hspace{-1ex} (b) $\vy$  &
\hspace{-1ex} (c) $\vx-\vy$  &
\hspace{-1ex} (d) $\calD_{\sigma}(\vx)$ &
\hspace{-1ex} (e) $\calD_{\sigma}(\vy)$ &
\hspace{-1ex} (f) $\calD_{\sigma}(\vx)-\calD_{\sigma}(\vy)$
\end{tabular}
\caption{Counter example showing non-local means is expansive. $\kappa = \|\calD_{\sigma}(\vx) - \calD_{\sigma}(\vy)\|^2/\|\vx-\vy\|^2 = 1.1775$.}
\label{fig:counter}
\vspace{-2ex}
\end{figure*}

\section{Conclusion}
We presented a continuation scheme for Plug-and-Play ADMM. We showed that for any \emph{bounded denoisers} (denoisers that asymptotically converges to the identity operator), the new Plug-and-Play ADMM has a provable fixed point convergence. We demonstrated two applications of the new algorithm for single image super-resolution and the single photon imaging problem. For the single image super-resolution problem, we presented a closed-form approach to solve one of the two subproblems in the ADMM algorithm. The closed-form result allows significantly faster implementation than iterative methods. Experimentally, we found that Plug-and-Play ADMM performs better than several existing methods.

\section*{Acknowledgement}
We thank Professor Charles Bouman and Suhas Sreehari for useful discussions. We thank Professor Gregery Buzzard for the constructive critique about the proof of Theorem 1. We thank the anonymous reviewers for the useful feedback which significantly improves the paper.

\appendices
\section{Counter Example of Non-Expansive Denoiser}

As mentioned in Section II.B, showing non-expansiveness of a denoiser could be difficult. Here we provide a counter example for the non-local means \cite{Buades_Coll_2005_Journal}.

To show that non-local means are expansive, we only need to find a pair $(\vx,\vy)$ such that
\begin{equation*}
\kappa = \|\calD_{\sigma}(\vx) - \calD_{\sigma}(\vy)\|^2/\|\vx-\vy\|^2 > 1.
\end{equation*}
To construct such example, we show in \fref{fig:counter} a pair of $(\vx,\vy)$ obtained through an inpainting problem using Plug-and-Play ADMM with constant $\rho$, i.e., $\gamma = 1$. (In fact, it does not matter how we obtain this pair of $(\vx,\vy)$. All we need to show is that there exists $(\vx,\vy)$ which makes $\kappa > 1$.)

The non-local means is a weighted average operation with weights
\begin{equation*}
W_{ij} = \exp\{-\|\vx_i-\vx_j\|^2/(2\sigma^2)\},
\end{equation*}
where $\vx_i$ is the $i$th patch of the image $\vx$. To further ensure that $\mW$ is doubly stochastic so that its eigenvalues are bounded between 0 and 1, we apply Sinkhorn-Knopp \cite{Sinkhorn_1964} to $\mW$ until convergence. Define this doubly stochastic matrix as $\mWtilde$. Then, the denoised output is given by
\begin{equation*}
\vxtilde = \calD_{\sigma}(\vx) \bydef \mWtilde \vx.
\end{equation*}
Therefore, the ratio we need to check is
\begin{equation*}
\kappa = \|\mWtilde_{\vx}(\vx) - \mWtilde_{\vy}(\vy)\|^2/\|\vx-\vy\|^2,
\end{equation*}
where the subscript $(\cdot)_{\vx}$ specifies the dependency of $\mWtilde$ on $\vx$ (or $\vy$). The denoised results are shown in \fref{fig:counter} (c) and (d). Although it may look subtle, one can verify that $\kappa = 1.1775$ which violates the requirement of non-expansiveness. This happens because $\mWtilde_{\vx} \not= \mWtilde_{\vy}$ for $\vx \not= \vy$. The dependency on $\vx$ and $\vy$ makes the operators nonlinear, and hence makes non-expansiveness difficult to validate.

Readers at this point may wonder why the proposed Plug-and-Play ADMM can alleviate the expansive issue. In a nutshell, the reason is that we force $\rho \rightarrow \infty$ so that $\sigma \rightarrow 0$. Consequently, the weight $\mW \rightarrow \mI$ as $\rho \rightarrow \infty$. For the original Plug-and-Play ADMM in \cite{Venkatakrishnan_Bouman_Wohlberg_2013}, $\mW \nrightarrow \mI$ because $\rho$ is fixed.

\section{Proof of Theorem 1}
To simplify the notations we first define a triplet $\vtheta^{(k)} \bydef (\vx^{(k)}, \, \vv^{(k)}, \, \vu^{(k)})$. Let $\mTheta$ be the domain of $\vtheta^{(k)}$ for all $k$. On $\mTheta$ we define a distance function $D: \mTheta \times \mTheta \rightarrow \R$ such that
\begin{align*}
D(\vtheta^{(k)}, \vtheta^{(j)})
&= \frac{1}{\sqrt{n}}\Big(\|\vx^{(k)}-\vx^{(j)}\|_2 + \|\vv^{(k)}-\vv^{(j)}\|_2 \\
&\quad\quad + \|\vu^{(k)}-\vu^{(j)}\|_2\Big).
\end{align*}
It then follows that $\Delta_{k+1} = D(\vtheta^{(k+1)}, \vtheta^{(k)})$. Since $\mTheta \subseteq \R^{3n}$ and $\R^{3n}$ is a complete metric space, as long as we can show that $\{\vtheta^{(k)}\}_{k=1}^{\infty}$ is a Cauchy sequence in $\mTheta$ with the distance function $D$, then $\vtheta^{(k)}$ should converge.

The Plug-and-Play ADMM involves two cases of the parameter update:
\begin{itemize}
\item Case 1: If $\Delta_{k+1} > \eta \Delta_{k}$, then $\rho_{k+1} = \gamma \rho_k$.
\item Case 2: If $\Delta_{k+1} \le \eta \Delta_{k}$, then $\rho_{k+1} = \rho_k$.
\end{itemize}
At iteration $k$, if Case 1 holds, then by Lemma 1, $\vtheta^{(k+1)}$ satisfies
\begin{equation}
D(\vtheta^{(k+1)},\vtheta^{(k)}) \le \frac{C'}{\sqrt{\rho_k}},
\label{eq:bound 1}
\end{equation}
for some universal constant $C' > 0$ independent of $k$. On the other hand, if Case 2 holds, then since $\Delta_{k+1} = D(\vtheta^{(k+1)}, \vtheta^{(k)})$ we have
\begin{equation}
D(\vtheta^{(k+1)},\vtheta^{(k)}) \le \eta D(\vtheta^{(k)},\vtheta^{(k-1)}).
\end{equation}

As $k \rightarrow \infty$, one of the following situations will happen:
\begin{enumerate}
\item[($S_1$)]: Case 1 occurs infinitely many times but Case 2 occurs finitely many times;
\item[($S_2$)]: Case 2 occurs infinitely many times but Case 1 occurs finitely many times;
\item[($S_3$)]: Both Case 1 and Case 2 occur infinitely many times.
\end{enumerate}

These three cases can be analyzed as follows. When ($S_1$) happens, there must exists a $K_1$ such that for $k \ge K_1$ only Case 1 will be visited. Thus,
\begin{align*}
D(\vtheta^{(k+1)},\vtheta^{(k)}) \le \frac{C'}{\sqrt{\rho_{K_1-1}} \sqrt{\gamma}^{k-K_1}}.
\end{align*}
When ($S_2$) happens, there must exists a $K_2$ such that for $k \ge K_2$ only Case 2 will be visited. Thus, we have
\begin{align*}
D(\vtheta^{(k+1)},\vtheta^{(k)})
&\le \eta^{k-K_2} D(\vtheta^{(K_2)},\vtheta^{(K_2-1)})\\
&\le \eta^{k-K_2} \frac{C'}{\rho_{K_2-1}}.
\end{align*}
($S_3$) is a union o the ($S_1$) and ($S_2$). Therefore, as long as we can show under ($S_1$) and ($S_2$) the sequence $\{\vtheta^{(k)}\}_{k=1}^{\infty}$ converges, the sequence will also converge under ($S_3$). To summarize, we show in Lemma 2 that regardless which of ($S_1$)-($S_3$), for any $k$ we have
\begin{equation*}
D(\vtheta^{k+1}, \vtheta^{(k)}) \le C'' \delta^{k},
\end{equation*}
for some constants $C''$ and $0 < \delta < 1$. Therefore,
\begin{equation}
D(\vtheta^{(k+1)},\vtheta^{(k)}) \rightarrow 0,
\label{eq:D_k}
\end{equation}
as $k \rightarrow \infty$.

To prove $\{\vtheta^{(k)}\}_{k=1}^{\infty}$ is a Cauchy sequence, we need to show
\begin{equation}
D(\vtheta^{(m)},\vtheta^{(k)}) \rightarrow 0,
\end{equation}
for all integers $m>k$ and $k \rightarrow \infty$. This result holds because for any finite $m$ and $k$,
\begin{align*}
D(\vtheta^{(m)},\vtheta^{(k)})
&\le \sum_{n=k+1}^{m} C'' \delta^n \\
&= \sum_{\ell = 1}^{m-k} C'' \delta^{\ell+k}\\
&= C''\delta^k \frac{1-\delta^{m-k+1}}{1-\delta}.
\end{align*}
Therefore, as $k\rightarrow \infty$, $D(\vtheta^{(m)},\vtheta^{(k)}) \rightarrow 0$. Hence, $\{\vtheta^{(k)}\}_{k=1}^{\infty}$ is a Cauchy sequence. Since a Cauchy sequence in $\R^{3n}$ always converges, there must exists $\vtheta^{*} = (\vx^*, \vv^*, \vu^*)$ such that
\begin{equation}
D(\vtheta^{(k)},\vtheta^*) \rightarrow 0.
\end{equation}
Consequently, we have $\|\vx^{(k)} - \vx^*\|_2 \rightarrow 0$, $\|\vv^{(k)} - \vv^*\|_2 \rightarrow 0$ and $\|\vu^{(k)} - \vu^*\|_2 \rightarrow 0$. This completes the proof.

\begin{lemma}
At iteration $k$, if Case 1 holds, then
\begin{equation}
D(\vtheta^{(k+1)},\vtheta^{(k)}) \le \frac{C'}{\sqrt{\rho_k}},
\end{equation}
for some universal constant $C' > 0$ independent of $k$.
\end{lemma}

\begin{proof}
Following the definition of $D(\vtheta^{(k+1)},\vtheta^{(k)})$, it is sufficient to show that
\begin{align*}
\frac{1}{\sqrt{n}}\left\| \vx^{(k+1)} - \vx^{(k)}\right\|_{2} &\le \frac{C_1}{\sqrt{\rho_k}},\\
\frac{1}{\sqrt{n}}\left\| \vv^{(k+1)} - \vv^{(k)}\right\|_{2} &\le \frac{C_2}{\sqrt{\rho_k}},\\
\frac{1}{\sqrt{n}}\left\| \vu^{(k+1)} - \vu^{(k)}\right\|_{2} &\le \frac{C_3}{\sqrt{\rho_k}},
\end{align*}
for some universal constants $C_1$, $C_2$ and $C_3$.

Let us consider
\begin{equation*}
\vx^{(k+1)} = \argmin{\vx} \; f(\vx) + \frac{\rho_k}{2}\|\vx - (\vv^{(k)}-\vu^{(k)})\|^2.
\end{equation*}
The first order optimality implies that
\begin{equation*}
\vx - (\vv^{(k)}-\vu^{(k)}) = -\frac{1}{\rho_k} \nabla f(\vx).
\end{equation*}
Since the minimizer is $\vx = \vx^{(k+1)}$, substituting $\vx = \vx^{(k+1)}$ and using the fact that $\nabla f$ is bounded yields
\begin{equation}
\frac{1}{\sqrt{n}}\left\|\vx^{(k+1)} - (\vv^{(k)}-\vu^{(k)})\right\|_{2} = \frac{\|\nabla f(\vx)\|_2}{\rho_k\sqrt{n}} \le \frac{L}{\rho_k}.
\label{eq:proof bound 1}
\end{equation}

Next, let $\vvtilde^{(k)} = \vx^{(k+1)}+\vu^{(k)}$ and $\sigma_k = \sqrt{\lambda/\rho_k}$. Define
\begin{equation*}
\vv^{(k+1)} = \calD_{\sigma_k}(\vvtilde^{(k)}).
\end{equation*}
Since $\calD_{\sigma_k}$ is a bounded denoiser, we have that
\begin{align}
& \frac{1}{\sqrt{n}}\left\|\vv^{(k+1)} - (\vx^{(k+1)}+\vu^{(k)})\right\|_{2}
= \frac{1}{\sqrt{n}}\left\|\vv^{(k+1)} - \vvtilde^{(k)}\right\|_{2} \notag\\
&= \frac{1}{\sqrt{n}}\left\|\calD_{\sigma_k}(\vvtilde^{(k)}) - \vvtilde^{(k)}\right\|_{2} \le \sigma_k \sqrt{C} = \frac{\sqrt{\lambda}\sqrt{C} }{\sqrt{\rho_k}}.
\label{eq:proof bound 2}
\end{align}

We can now bound $\left\|\vv^{(k+1)} - \vv^{(k)}\right\|_{2}$ as follows.
\begin{align}
\frac{1}{\sqrt{n}}\left\|\vv^{(k+1)} - \vv^{(k)}\right\|_{2}
&\le \frac{1}{\sqrt{n}}\left\|\vv^{(k+1)} - (\vx^{(k+1)}+\vu^{(k)}) \right\|_{2} \notag\\
&\quad+ \frac{1}{\sqrt{n}}\left\|(\vx^{(k+1)}+\vu^{(k)}) - \vv^{(k)}\right\|_{2}.\notag
\end{align}
Using \eref{eq:proof bound 1} and \eref{eq:proof bound 2}, we have
\begin{align}
&\frac{1}{\sqrt{n}}\left\|\vv^{(k+1)} - \vv^{(k)}\right\|_{2} \le \frac{\sqrt{\lambda}\sqrt{C} }{\sqrt{\rho_k}} + \frac{L}{\rho_k}\\
&= \frac{1}{\sqrt{\rho_k}}\left(\sqrt{\lambda}\sqrt{C} + \frac{L}{\sqrt{\rho_k}}\right) \le  \frac{1}{\sqrt{\rho_k}}\left(\sqrt{\lambda}\sqrt{C} + \frac{L}{\sqrt{\rho_0}}\right). \notag
\end{align}
Similarly, we can show that
\begin{align}
&\frac{1}{\sqrt{n}}\left\|\vu^{(k+1)} \right\|_{2} = \frac{1}{\sqrt{n}}\left\|\vu^{(k)} + (\vx^{(k+1)} - \vv^{(k+1)})\right\|_{2}\notag \\
&= \frac{1}{\sqrt{n}}\left\|\vu^{(k)} + \vx^{(k+1)} - \calD_{\sigma_k}(\vvtilde^{(k)})\right\|_{2}\notag \\
&= \frac{1}{\sqrt{n}}\left\|\vu^{(k)} + \vx^{(k+1)} - (\calD_{\sigma_k}(\vvtilde^{(k)}) - \vvtilde^{(k)}) - \vvtilde^{(k)}\right\|_{2}\notag \\
&\overset{(a)}{=} \frac{1}{\sqrt{n}}\left\|\calD_{\sigma_k}(\vvtilde^{(k)}) - \vvtilde^{(k)}\right\|_{2}
\le \frac{\sqrt{\lambda}\sqrt{C}}{\sqrt{\rho_k}},
\end{align}
where (a) holds because $\vvtilde^{(k)} = \vu^{(k)} + \vx^{(k+1)}$. Thus,
\begin{align*}
\frac{1}{\sqrt{n}}\left\| \vu^{(k+1)} - \vu^{(k)}\right\|_{2}
&\le \frac{1}{\sqrt{n}}\left(\left\| \vu^{(k+1)}\right\|_{2} + \left\|\vu^{(k)}\right\|_{2}\right)\\
&\le \frac{2\sqrt{\lambda}\sqrt{C} }{ \sqrt{\rho_k}}.
\end{align*}
Finally, since $\vu^{(k+1)} = \vu^{(k)} + (\vx^{(k+1)}-\vv^{(k+1)})$, we have
\begin{align}
&\frac{1}{\sqrt{n}}\left\| \vx^{(k+1)} - \vx^{(k)}\right\|_{2}\notag \\
&= \frac{1}{\sqrt{n}}\left\| \left(\vu^{(k+1)}-\vu^{(k)}+\vv^{(k+1)}\right) - \left(\vu^{(k)}-\vu^{(k-1)}+\vv^{(k)}\right) \right\|_{2} \notag \\
&\le \frac{1}{\sqrt{n}}\Big(\left\| \vu^{(k+1)}-\vu^{(k)} \right\|_{2} + \left\| \vu^{(k)}-\vu^{(k-1)} \right\|_{2} \notag\\
&\quad\quad\quad + \left\|\vv^{(k+1)} - \vv^{(k)}\right\|_{2}\Big)\notag \\
&\le \frac{2 \sqrt{\lambda} \sqrt{C}}{\sqrt{\rho_k}} + \frac{2 \sqrt{\lambda} \sqrt{C}}{ \sqrt{\rho_{k-1}}} + \frac{1}{\sqrt{\rho_k}}\left(\sqrt{\lambda}\sqrt{C} + \frac{L}{\sqrt{\rho_k}}\right)\notag\\
&\le \left((3+2\sqrt{\gamma})\sqrt{\lambda} \sqrt{C} + \frac{L}{\sqrt{\rho_0}}\right)\frac{1}{\sqrt{\rho_k}}.
\end{align}
\end{proof}

\begin{lemma}
The sequence $\{\vtheta^{(k)}\}_{k=1}^{\infty}$ always satisfies
\begin{equation}
D(\vtheta^{k+1}, \vtheta^{(k)}) \le C'' \delta^{k},
\end{equation}
for some constants $C''$ and $0 < \delta < 1$.
\end{lemma}

\begin{proof}
At any iteration $k$, it holds that
\begin{align*}
D(\vtheta^{k+1}, \vtheta^{(k)})
&\le \max\left( \frac{C'}{\sqrt{\rho_{K_1-1}} \sqrt{\gamma}^{k-K_1}}, \eta^{k-K_2} \frac{C'}{\rho_{K_2-1}}\right)\\
&\le \max\left( C'_1 \left(\frac{1}{\sqrt{\gamma}}\right)^{k}, C'_2 \eta^{k} \right),
\end{align*}
where $C'_1 = C' \sqrt{ \frac{ \gamma^{K_1}}{\rho_{K_1-1}}}$ and $C'_2 = \frac{C'\eta^{-K_2}}{\rho_{K_2-1}}$. Therefore, by letting
\begin{align*}
C'' = \max \; (C'_1, C'_2), \quad\mbox{and}\quad \delta = \max \; ( 1/\sqrt{\gamma}, \eta),
\end{align*}
we obtain the desired result, as $\gamma > 1$.
\end{proof}

\bibliographystyle{IEEEbib}
\bibliography{refs}

\end{document}